\documentclass{article}

\usepackage{graphicx}
\usepackage{booktabs} % for professional tables
\usepackage{xspace}
\usepackage{wrapfig}
\usepackage{enumitem}
\usepackage{xargs}   
\usepackage{array}
\usepackage{subcaption}
\usepackage{hyperref}

\usepackage[accepted]{icml2023}
\newcommand{\cmnist}{Color-MNIST\xspace}
\newcommand{\mnistcifar}{MNIST-CIFAR\xspace}
\newcommand{\mnist}{MNIST\xspace}
\newcommand{\cifar}{CIFAR\xspace}
\newcommand{\rank}[1]{\textrm{rank}\left(#1\right)}
\newcommand{\set}[1]{\left\{#1\right\}}
\usepackage{amsmath,amsfonts,bm}
\newcommand{\cD}{\mathcal{D}}
\newcommand{\ifm}{IFM\xspace}
\newcommand{\defeq}{\stackrel{\textrm{def}}{=}}
\newcommand{\trace}[1]{\textrm{Tr}\left(#1\right)}
\newcommand{\xtilde}{\widetilde{x}}
\newcommand{\ce}[1]{\gL\left(#1\right)}
\newcommand{\mdiv}[1]{\textrm{Mist-Div}\left(#1\right)}
\newcommand{\cclc}[1]{\textrm{CC-LogitCorr}\left(#1\right)}
\newcommand{\corr}[1]{\textrm{Corr}\left(#1\right)}

\newcommand{\norm}[1]{\left\| #1 \right\|}
\newcommand{\ft}{\tilde{f}}
\newcommand{\ldsb}{LD-SB\xspace}
\newcommand{\diverse}{OrthoP\xspace}
\def\eqref#1{eqn.~(\ref{#1})}
\def\Eqref#1{Eqn.~(\ref{#1})}
\def\1{\bm{1}}

\DeclareMathAlphabet{\mathsfit}{\encodingdefault}{\sfdefault}{m}{sl}
\SetMathAlphabet{\mathsfit}{bold}{\encodingdefault}{\sfdefault}{bx}{n}

\def\gD{{\mathcal{D}}}

\def\gH{{\mathcal{H}}}

\def\gL{{\mathcal{L}}}

\def\gN{{\mathcal{N}}}

\def\gP{{\mathcal{P}}}

\def\gR{{\mathcal{R}}}
\def\gS{{\mathcal{S}}}

\def\sN{{\mathbb{N}}}

\def\sR{{\mathbb{R}}}
\def\sS{{\mathbb{S}}}

\newcommand{\E}{\mathbb{E}}

\newcommand{\R}{\mathbb{R}}

\newcommand{\iprod}[2]{\langle #1, #2 \rangle}

\DeclareMathOperator*{\argmax}{arg\,max}
\DeclareMathOperator*{\argmin}{arg\,min}

\usepackage{amsmath}
\usepackage{amssymb}
\usepackage{mathtools}
\usepackage{amsthm}

\usepackage[capitalize,noabbrev]{cleveref}

\newcommandx{\pj}[2][1=]{\todo[linecolor=red,backgroundcolor=red!25,bordercolor=red,#1]{#2}}

\newcommand{\x}{\mathbf{x}}

\newcommand{\e}{\mathbf{e}}

\renewcommand{\L}{\mathcal L}

\newcommand{\Y}{\mathcal Y}
\newcommand{\U}{\mathcal U}
\newcommand{\F}{\mathcal F}
\newcolumntype{M}[1]{>{\centering\arraybackslash}m{#1}}

\theoremstyle{plain}
\newtheorem{theorem}{Theorem}[section]
\newtheorem{proposition}[theorem]{Proposition}
\newtheorem{lemma}[theorem]{Lemma}

\theoremstyle{definition}
\newtheorem{definition}[theorem]{Definition}

\theoremstyle{remark}

\newtheorem{claim}[theorem]{Claim}
\usepackage[textsize=tiny]{todonotes}

\icmltitlerunning{Simplicity Bias in 1-Hidden Layer Neural Networks}

\begin{document}

\twocolumn[
\icmltitle{Simplicity Bias in 1-Hidden Layer Neural Networks}




\begin{icmlauthorlist}
\icmlauthor{Depen Morwani }{yyy}
\icmlauthor{Jatin Batra}{comp}
\icmlauthor{Prateek Jain}{sch,alp}
\icmlauthor{Praneeth Netrapalli}{sch,alp}
\end{icmlauthorlist}

\icmlaffiliation{yyy}{Department of Computer Science, Harvard University, Cambridge, MA, USA (part of the work done while at Google Research, Bengaluru, India)}
\icmlaffiliation{comp}{School of Technology and Computer Science,
TIFR,
Mumbai, India}
\icmlaffiliation{sch}{Google Research, Bengaluru, India.}
\icmlaffiliation{alp}{Alphabetical ordering}
\icmlcorrespondingauthor{Depen Morwani}{dmorwani@g.harvard.edu}
%\icmlcorrespondingauthor{Firstname2 Lastname2}{first2.last2@www.uk}

\icmlkeywords{Machine Learning, ICML}

\vskip 0.3in
]



\printAffiliationsAndNotice{}  % leave blank if no need to mention equal contribution

\begin{abstract}
% This document provides a basic paper template and submission guidelines.
% Abstracts must be a single paragraph, ideally between 4--6 sentences long.
% Gross violations will trigger corrections at the camera-ready phase.
Recent works~\cite{shah2020pitfalls,chen2021intriguing} have demonstrated that neural networks exhibit extreme \emph{simplicity bias} (SB). That is,  they learn \emph{only the simplest} features  to solve a task at hand, even in the presence of other, more robust but more complex features. Due to the lack of a general and rigorous definition of \emph{features}, these works showcase SB on \emph{semi-synthetic} datasets such as \cmnist , \mnistcifar where defining features is relatively easier. 

In this work, we rigorously define as well as thoroughly establish SB for \emph{one hidden layer} neural networks. More concretely, (i) we define SB as the network essentially being a function of a low dimensional projection of the inputs 
% (the weight matrices may or may not be low rank)
%characterize SB in terms of rank of learned weight matrices of the network -- low-rank implies that the network depends only on a low dimensional projection of the input, 
(ii) theoretically, we show that when the data is linearly separable, the network primarily depends on only the linearly separable ($1$-dimensional) subspace even in the presence of an arbitrarily large number of other, more complex features which could have led to a significantly more robust classifier,  (iii) empirically, we show that models trained on \emph{real} datasets such as Imagenette and Waterbirds-Landbirds indeed depend on a low dimensional projection of the inputs, thereby demonstrating SB on these datasets, iv) finally, we present a natural ensemble approach that encourages diversity in  models by training successive models on features not used by earlier models, and demonstrate that it yields models that are significantly more robust to Gaussian noise.
\end{abstract}

\section{Introduction}
% \begin{itemize}
%     \item OOD robustness, feature learning and simplicity bias.
%     \item Simplicity bias and follow up works.
%     \item What is the precise mathematical definition of simplicity bias?
%     \item The low rank phenomenon
%     \item Contributions: low rank characterization, theoretical results, empirical results and a method to fix it.
% \end{itemize}
\begin{figure}
%   \begin{center}
    \includegraphics[width=0.48\textwidth]{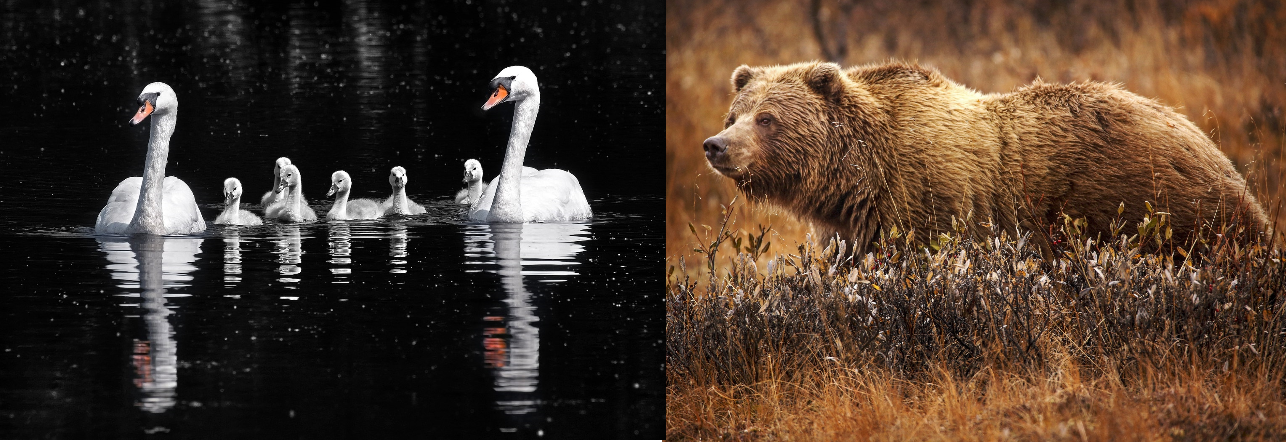}
%   \end{center}
  \caption[swansbears]{Classification of swans vs bears. There are several features such as background, color of the animal, shape of the animal etc., each of which is sufficient for classification but using all of them will lead to a more robust model. \footnotemark}
  \label{fig:swans-bears}
  %\vspace{-0.5cm}
\end{figure}
It is well known that neural networks (NNs) are vulnerable to distribution shifts as well as to adversarial examples~\citep{szegedy2013intriguing,hendrycks2021many}. A recent line of work~\citep{geirhos2018imagenet,shah2020pitfalls,geirhos2020shortcut} proposes that \emph{Simplicity Bias (SB)} -- aka shortcut learning -- i.e., the tendency of neural networks (NNs) to learn only the simplest features over other useful but more complex features, is a key reason behind this non-robustness. 
\footnotetext{Image source: Wikipedia~\cite{swan},~\cite{bear}.} The argument is roughly as follows: for example, in the classification of swans vs bears, as illustrated in Figure~\ref{fig:swans-bears},
% \pj{can we add more detail about the wiki source}\pj{the picture is lifted from our talks, can that be an issue?},
there are many features such as background, color of the animal, shape of the animal etc. that can be used for classification. However using only one or few of them can lead to models that are not robust to specific distribution shifts, while using all the features can lead to more robust models.

Several recent works have demonstrated SB on a variety of {\em semi-real constructed datasets}   \citep{geirhos2018imagenet,shah2020pitfalls,chen2021intriguing}, and have hypothesized SB to be the key reason for NN's brittleness to distribution shifts \citep{shah2020pitfalls}. However, such observations are still only for specific semi-real datasets, and a general method that can identify SB on a \emph{given dataset} and a \emph{given model} is still missing in literature. Such a method would be useful not only to estimate the robustness of a model but could also help in designing more robust models. 

A key challenge in designing such a general method to identify (and potentially fix) SB is that the notion of \emph{feature} itself is vague and lacks a rigorous definition. Existing works like \cite{geirhos2018imagenet, shah2020pitfalls,chen2021intriguing} avoid this challenge of vague  feature definition by using carefully designed datasets (e.g., concatenation of \mnist images and \cifar images), where certain high level features (e.g., \mnist features and \cifar features, shape and texture features) are already baked in the dataset definition, and arguing about their {\em simplicity} is intuitively easy.
% Thus, the papers are also able to argue about the simplicity bias of DNN training. 
% In this paper, we address this question by 
% they do not provide a \emph{general method} that can be used to detect whether a trained model on \emph{a given} dataset suffers from SB (so that we know if the model is not robust) and if so, how to make it more robust. \praneeth{Example motivation}: For example, given a model trained to predict the prevalence of a disease from x-ray images, it would be most useful to have a method that can flag if the model might be relying on simple or shortcut features. Further if it can also identify these simple/shortcut features and guide the training of a more robust model, that would be extremely useful.
% However, this latter question is especially important both as a measure of robustness of a model as well as in \emph{robustifying} the model.

\textbf{Contributions}: One of the main contributions of this work is to provide a precise definition of a particular  simplicity bias -- \ldsb -- of \emph{$1$-hidden layer neural networks}. In particular, we characterize SB as \emph{low dimensional input dependence} of the model. Concretely, 
\begin{definition}[\ldsb]\label{def:ldsb}
% \textbf{SB through low dimensional input dependence} (Informal):
A model $f:\R^d \rightarrow \R^c$ with inputs $x \in \R^d$ and outputs $f(x) \in \R^c$ (e.g., logits for $c$ classes), trained on a distribution ${(x,y) \sim \cD}$ satisfies \ldsb if there exists a \emph{projection} matrix $P \in \R^{d \times d}$ satisfying:
\begin{itemize}[leftmargin=0.5cm]
    \item $\rank{P} =k \ll d$,
    \item $f(P x^{(1)} + P_\perp x^{(2)}) \approx f(x_1) \; \forall (x^{(1)}, y^{(1)})$, $(x^{(2)}, y^{(2)}) \sim \cD$
    \item An independent model $g$ trained on $(P_\perp x, y)$ where $(x,y)\sim \cD$ achieves high accuracy.
\end{itemize}
Here $P_{\perp}$ is the projection matrix onto the subspace orthogonal to $P$.
\end{definition}
In words, \ldsb says that there exists a small $k$-dimensional subspace (given by the projection matrix $P$) in the input space $\R^d$, which is the only thing that the model $f$ considers in labeling any input point $x$. In particular, if we \emph{mix} two data points $x_1$ and $x_2$ by using the projection of $x_1$ onto $P$ and the projection of $x_2$ onto the orthogonal subspace $P_{\perp}$, the output of $f$ on this \emph{mixed point} $Px_1 + P_\perp x_2$ is the same as that on $x_1$. This would have been fine if the subspace $P_\perp$ does not contain any feature useful for classification. However, the third bullet point says that $P_\perp$ indeed contains features that are useful for classification since an independent model $g$ trained on $(P_\perp x, y)$ achieves high accuracy.

Furthermore, theoretically, we demonstrate  \ldsb of $1$-hidden layer NNs for a fairly general class of distributions called \emph{independent features model (\ifm)}, where the features (i.e., coordinates) are distributed independently conditioned on the label. \ifm has a long history and is widely studied, especially in the context of naive-Bayes classifiers~\cite{lewis1998naive}.
% for some fairly simple but rich datasets.
% In particular, we consider distributions 
% binary classification under  dataset where {\em one} coordinate can separate out the two classes, while rest of the $d-1$ coordinates require 2-piecewise classifier to separate the points; $d\gg 1$. Here,
For \ifm, we show that as long as there is even a \emph{single} feature in which the data is linearly separable, NNs trained using SGD will learn models that rely almost exclusively on this linearly separable feature, even when there are an \emph{arbitrarily large number} of features in which the data is separable but with a \emph{non-linear} boundary. Empirically, we demonstrate \ldsb on three real world datasets: binary and multiclass version of Imagenette \citep{imagenette} as well as waterbirds-landbirds \citep{Sagawa*2020Distributionally} dataset. Compared to the results in~\cite{shah2020pitfalls}, our results (i) theoretically show \ldsb in a fairly general setting and (ii) empirically show \ldsb on real datasets.
% trained in both {\em rich} and {\em lazy} initialization settings would lead to extreme \ldsb. Thus establishing \ldsb as a fundamental challenge for training NNs. Moreover, we demonstrate \ldsb for real-world datasets like Imagenette \cite{} where the features are learned using a standard pre-trained network, which are then consumed by a $1$-hidden layer NN to predict the label. \praneeth{Contrast with \cite{shah2020pitfalls} and highlight only product distribution assumption but nothing else and even if infinitely many complex features present.}

Finally, building upon these insights, we propose a simple ensemble method -- \emph{\diverse} -- that sequentially constructs NNs by projecting out principle input data directions that are used by previous NNs. We demonstrate that this method can lead to significantly more robust ensembles for real-world datasets in presence of simple distribution shifts like Gaussian noise.  

\textbf{Why only $1$-hidden layer networks?}: One might wonder why the results in this paper are restricted to $1$-hidden layer networks and why they are interesting. We present two reasons.
\begin{enumerate}[leftmargin=*,noitemsep,nolistsep]
    \item From a \textbf{theoretical} standpoint, prior works have thoroughly characterized the training dynamics of infinite width $1$-hidden layer networks under different initialization schemes~\citep{chizat2019lazy} and have also identified the limit points of gradient descent for such networks~\citep{ChizatB20}. Our results crucially build upon these prior works. On the other hand, we do not have such a clear understanding of the dynamics of deeper networks \footnote{For more discussion on the difficulty of extending these results to deep nets, refer Appendix \ref{app:deep-nets}}.
    \item From a \textbf{practical} standpoint, the dominant paradigm in machine learning right now is to pretrain large models on large amounts of data and then finetune on small target datasets. Given the large and diverse pretraining data seen by these models, it has been observed that they do learn rich features~\citep{rosenfeld2022domain,nasery2022daft}. However, finetuning on target datasets might not utilize all the features in the pretrained model. Consequently, approaches that can train robust finetuning heads (such as a $1$-hidden layer network on top) can be quite effective.
\end{enumerate}
Extending our results to deeper networks and to other architectures is an exciting direction of research from both theoretical and practical points of view.
% We summarize the key contributions of the paper below: 
% \begin{enumerate}
%     \item Under the popular \emph{independent features model} for data, we \emph{theoretically characterize} the models learned by gradient descent. This reveals that SB can be characterized by the \emph{low dimensional input dependence} of the model -- we refer to this as \ldsb.
%     \item We use this insight to empirically demonstrate that models learned on several \emph{real} datasets indeed suffer from \ldsb.
%     \item Finally, we use this insight to design a new algorithm, \emph{\diverse}, that can learn a second model which relies on features that the first model does not use. We show that an ensemble of the two models has higher robustness (to Gaussian noise) compared to an ensemble of two independently trained models.
% \end{enumerate}

\textbf{Paper organization}: This paper is organized as follows. Section~\ref{sec:related} presents related work. Section~\ref{sec:prelims} presents preliminaries. Our main results on \ldsb are presented in Section~\ref{sec:results}. Section~\ref{sec:diverse} presents results on training diverse classifiers. We conclude in Section~\ref{sec:disc}.
\section{Related Work}\label{sec:related} 
In this section, we briefly mention the closely related works. Extended related work can be found in Appendix \ref{app:rel-works}.

\textbf{Simplicity Bias}: Subsequent to~\cite{shah2020pitfalls}, there have been several papers investigating the presence/absence of SB in various networks as well as reasons behind SB~\cite{scimeca2021shortcut}. Of these,~\cite{huh2021low} is the most closely related work to ours. ~\cite{huh2021low} \emph{empirically observe} that on certain \emph{synthetic} datasets, the \emph{embeddings} of NNs both at initialization as well as after training have a low rank structure. In contrast, we prove \ldsb \emph{theoretically} on the \ifm model as well as empirically validate this on \emph{real} datasets. Furthermore, our results show that while the \emph{network weights} exhibit low rank structure in the rich regime (see Section~\ref{prelim:init} for definition), the manifestation of \ldsb is far more subtle in lazy regime.
Moreover, we also show how to use \ldsb to train a second diverse model and combine it to obtain a robust ensemble.~\cite{Tomer22} provide a theoretical intuition behind the relation between various hyperparameters (such as learning rate, batch size etc.) and rank of learnt weight matrices, and demonstrate it empirically.~\cite{pezeshki2021gradient} propose that \emph{gradient starvation} at the beginning of training is a potential reason for SB in the lazy/NTK regime but the conditions are hard to interpret. In contrast, our results are shown for any dataset in the \ifm model in the \emph{rich} regime of training. Finally~\cite{lyu2021gradient} consider anti-symmetric datasets and show that single hidden layer input homogeneous networks (i.e., without \emph{bias} parameters) converge to linear classifiers. However, our results hold for general datasets and do not require input homogeneity.
%However, such networks have strictly weaker expressive power compared to those with bias parameters.
% \todo{Do we have a reference for this?}

\textbf{Learning diverse classifiers}: There have been several works that attempt to learn diverse classifiers. Most works try to learn such models by ensuring that the input gradients of these models do not align~\citep{ross2018improving,teney2022evading}. \cite{xu2022controlling} propose a way to learn diverse/orthogonal classifiers under the assumption that a complete classifier, that uses all features is available, and demonstrates its utility for various downstream tasks such as style transfer. \cite{lee2022diversify} learn diverse classifiers by enforcing diversity on unlabeled target data.

\textbf{Spurious correlations}: There has been a large body of work which identifies reasons for spurious correlations in NNs~\citep{sagawa2020investigation} as well as proposing algorithmic fixes in different settings~\citep{liu2021just,chen2020self}.

\textbf{Implicit bias of gradient descent}: There is also a large body of work understanding the implicit bias of gradient descent dynamics. Most of these works are for standard linear~\citep{ji2019implicit} or deep linear networks~\citep{SoudryHNGS18,gunasekar2018implicit}. For nonlinear neural networks, one of the well-known results is for the case of $1$-hidden layer neural networks with homogeneous activation functions~\citep{ChizatB20}, which we crucially use in our proofs.

\section{Preliminaries}\label{sec:prelims}
In this section, we provide the notation and background on infinite width max-margin classifiers that is required to interpret the results of this paper.
% \praneeth{May be dataset description could go into Section~\ref{sec:theory}? In preliminaries, we can give definitions of $1$-hidden layer NNs as well as as well as related theory background on max margin etc? Thanks, sounds good -Jatin.}

%In this section, we define our model architecture, loss function and state important results which we will use repeatedly.

\subsection{Basic notions}

\textbf{1-hidden layer neural networks and loss function.}
Consider instances $x \in \gR^d$ and labels $y \in \{\pm 1\}$ jointly distributed as $\gD$.
A 1-hidden layer neural network model for predicting the label for a given instance $x$, is defined by parameters $(\bar{w} \in \sR^{m\times d},  \bar{b} \in \sR^m, \bar{a} \in \sR^m)$. For a fixed activation function $\phi$, given input instance $x$, the model is given as $f((\bar{w},\bar{b},\bar{a}),x) \coloneqq \langle \bar{a},\phi(\bar{w}x+\bar{b})\rangle$, where $\phi(\cdot)$ is applied elementwise. The cross entropy loss $\gL$ for a given model $f$, input $x$ and label $y$ is given as $\ce{f(x),y} \defeq \log(1+\exp(-yf((\bar{w},\bar{b},\bar{a}),x)))$. 

\noindent\textbf{Margin.} For data distribution $\gD$, the margin of a model $f(x)$ is given as $\min_{(x,y) \sim \gD} yf(x)$.

\noindent\textbf{Notation.} Here is some useful notation that we will use repeatedly. For a matrix $A$, $A(i,.)$ denotes the $i$th row of $A$. For any $k \in \sN$, $\sS^{k-1}$ denotes the surface of the unit norm Euclidean sphere in dimension $k$.

\subsection{Initializations} \label{prelim:init}
The gradient descent dynamics of the network depends strongly on the scale of initialization.
In this work, we primarily consider \emph{rich regime} initialization.
% s for the neural network - popularly known as the rich regime and lazy regime initialization.

\noindent\textbf{Rich regime.}
In rich regime initialization, for any $i \in [m]$, the parameters $(\bar{w}(i,.),\bar{b}(i)$) of the first layer are sampled from a uniform distribution on $\sS^d$. Each $\bar{a}(i)$ is sampled from $\textit{Unif}\{-1,1\}$, and the output of the network is scaled down by $\frac{1}{m}$ \citep{ChizatB20}. This is roughly equivalent to Xavier initialization~\citet{glorot2010understanding}, where the weight parameters in both the layers are initialized approximately as $\gN(0, \frac{2}{m})$ when $m \gg d$.

In addition, we also present some results for the lazy regime initialization described below.

\noindent\textbf{Lazy regime.}
In the lazy regime, the weight parameters in the first layer are initialized with $\gN(0, \frac{1}{d})$, those of second layer are initialized with $\gN(0, \frac{1}{m})$ and the biases are initialized to $0$~\citep{BiettiM19,LeeXSBNSP19}. This is approximately equivalent to Kaiming initialization~\citep{he2015delving}.
% ($\alpha$ can be $1$ or $2$ depending on whether we use Xavier or He initialization).
% In the lazy regime, we use the standard Xavier \citep{glorot2010understanding} or He initialization \citep{he2015delving} for the neural network. Basically, the weight parameters in both the layers are initialized with $\gN(0, \frac{\alpha}{m})$ and the biases are initialized to $0$ ($\alpha$ can be $1$ or $2$ depending on whether we use Xavier or He initialization).

\subsection{Infinite Width Case}

For 1-hidden layer neural networks with ReLU activation in the infinite width limit i.e., as $m \rightarrow \infty$, \citet{JacotGH18,chizat2019lazy,ChizatB20} gave interesting characterizations of the trained model. As mentioned above, the training process of these models falls into one of two regimes depending on the scale of initialization~\citep{chizat2019lazy}:
% It turns out that depending on the scale of initialization of the model, the training process of these models
% falls into one of two regimes~\cite{chizat2019lazy}: (i) rich or mean field regime, and (ii) lazy or kernel regime. 
%In the context of practical initialization schemes, for $1$-hidden layer neural networks, Xavier or Glorot initialization~\cite{glorot2010understanding} leads to rich regime while Kaiming or He initialization~\cite{he2015delving} leads to lazy regime.

\noindent\textbf{Rich regime.}
In the infinite width limit,  the neural network parameters can be thought of as a distribution $\nu$ over triples $(w,b,a)\in \sS^{d+1}$ where $w \in \sR^{d}, b,a \in \sR$.
%denote the weights of the first layer and the second layer, and the first layer bias for a single hidden layer neuron.
Under the rich regime initialization, the function $f$ computed by the model can be expressed as %and the respective cross entropy loss are
\begin{equation} \label{eq:rich:f} 
f(\nu,x) = \mathbb{E}_{(w,b,a) \sim \nu}[a(\phi(\langle w,x \rangle+b)]\,.
\end{equation}
%\quad\text{and}\quad \gL(\nu,(x,y)) = [\log(1+\exp(-yf(\nu,x))] \,.

\cite{ChizatB20} showed that the training process with rich initialization can be thought of as gradient flow on the Wasserstein-2 space and gave the following characterization \footnote{\noindent Theorem \ref{thm:CBrich} is an informal version of \citealt[Theorem 5]{ChizatB20}. For exact result, refer Theorem~\ref{thm:technical} in Appendix \ref{app:relu-F1}.}  of the trained model under the cross entropy loss $\mathbb{E}_{(x,y)\sim \gD} [\gL(\nu,(x,y))]$.

\begin{theorem}\citep{ChizatB20}\label{thm:CBrich}
Under rich initialization in the infinite width limit with cross entropy loss, if gradient flow on 1-hidden layer NN with ReLU activation converges, it converges to a maximum margin classifier $\nu^*$ given as
\begin{equation}\label{eq:nustar}
 \nu^* = \argmax_{\nu \in \gP(\sS^{d+1})} \min_{(x,y)\sim \gD} yf(\nu,x)\,, 
\end{equation}
where 
%$\sS^{d+1}$ denotes the surface of the $\ell_2$ unit ball in $\gR^{d+2}$ 
%and 
$\gP(\sS^{d+1})$ denotes the space of distributions over $\sS^{d+1}$.
\end{theorem}

This training regime is known as the `rich' regime since it learns data dependent features $\iprod{w}{\cdot}$.
% , in contrast with lazy regime. 

\textbf{Lazy regime.}
\cite{JacotGH18} showed that in the infinite width limit, the neural network behaves like a kernel machine. This kernel is popularly known as the Neural Tangent Kernel(NTK), and is given by $K(x, x') = \left\langle \frac{\partial f(x)}{\partial W}, \frac{\partial f(x')}{\partial W} \right\rangle$, where $W$ denotes the set of all trainable weight parameters. This initialization regime is called 'lazy' regime since the weights do not change much from initialization, and the NTK remains almost constant, i.e, the network does not learn data dependent features. 
%The NTK has been thoroughly studied since, for instance \cite{LiuZB20} shows that the NTK arises regardless of the exact training process as long as the weights don't change by much. 
We will use the following characterization of the NTK for 1-hidden layer neural networks.

\begin{theorem}\cite{BiettiM19}\label{thm:lazyBM}
Under lazy regime initialization in the infinite width limit, the NTK for 1-hidden layer neural networks with ReLU activation i.e., $\phi(u)= \max(u,0)$, is given as
\[ K(x,x') = \|x\|\|x'\|\kappa\left(\frac{\langle x,x' \rangle}{\|x\|\|x'\|}\right)\,,\]
where
\[\kappa(u)= \frac{1}{\pi}(2u(\pi - cos^{-1}(u))+\sqrt{1-u^2})\,.\]
%\[\kappa(u)=
%\begin{cases}
%\kappa_1(u) \coloneqq \frac{1}{\pi}(u(\pi - cos^{-1}(u))+\sqrt{1-u^2}) \text{ when only the second layer is trained},\\
%\kappa_2(u) \coloneqq \frac{1}{\pi}(2u(\pi - cos^{-1}(u))+\sqrt{1-u^2}) \text{ when both layers are trained}.
%\end{cases}
%\]
\end{theorem}

\emph{Lazy regime for binary classification.} \cite{SoudryHNGS18} showed that for linearly separable datasets, gradient descent for linear predictors on logistic loss converges to the max-margin support vector machine (SVM) classifier. This implies that, any sufficiently wide neural network, when trained for a finite time in the lazy regime on a dataset that is separable by the finite-width induced NTK, will tend towards the $\L_2$ max-margin-classifier given by
\begin{equation} \label{max-marg-ntk}
\argmin_{f \in \gH} \| f \|_{\gH} \text{ s.t. } yf(x) \geq 1 \text{ }\forall\text{ } (x,y) \sim \gD\,,    
\end{equation}
where $\gH$ represents the Reproducing Kernel Hilbert Space (RKHS) associated with the finite width kernel \citep{bachblog}. With increasing width, this kernel tends towards the infinite-width NTK (which is universal \citep{JiTX20}). Therefore, in lazy regime, we will focus on the $\L_2$ max-margin-classifier induced by the infinite-width NTK. 
% are motivated to analyze the $\L_2$ max-margin-classifer in the NTK regime given by

% \praneeth{Is there no result that extends the results of~\citet{SoudryHNGS18} to the kernel setting?}
%and the results of \cite{LiuZB20}, Theorem \ref{thm:lazyBM} allows us to use the max margin NTK SVM to characterize 1-hidden layer neural networks trained for binary classification in the lazy regime.

\section{Characterization of SB in $1$-hidden layer neural networks}\label{sec:results}
In this section, we first theoretically characterize the SB exhibited by gradient descent on linearly separable datasets in the \emph{independent features model (\ifm)}. The main result, stated in Theorem~\ref{thm:rich}, is that for binary classification of inputs in $\R^d$, even if there is a \emph{single} coordinate in which the data is linearly separable, gradient descent dynamics will learn a model that relies \emph{solely} on this coordinate, even when there are an arbitrarily large number $d-1$ of coordinates in which the data is separable, but by a non-linear classifier. In other words, the simplicity bias of these networks is characterized by \emph{low dimensional input dependence}, which we denote by \ldsb. We then experimentally verify that NNs trained on some real datasets do indeed satisfy \ldsb.

% which is described below. \praneeth{Highlight the results and section contents. Even with $1000$ slab features, model does not consider them.}

%\subsection{Theoretical Results}\label{sec:theory}
%In this section we present our main theoretical results demonstrating low dimensional SB for infinite width $1$-hidden layer neural networks on \ifm dataset.

\subsection{Dataset} \label{ifm:dataset}

\begin{figure}
\centering
\begin{subfigure}{0.3\columnwidth}
  \includegraphics[width=\columnwidth]{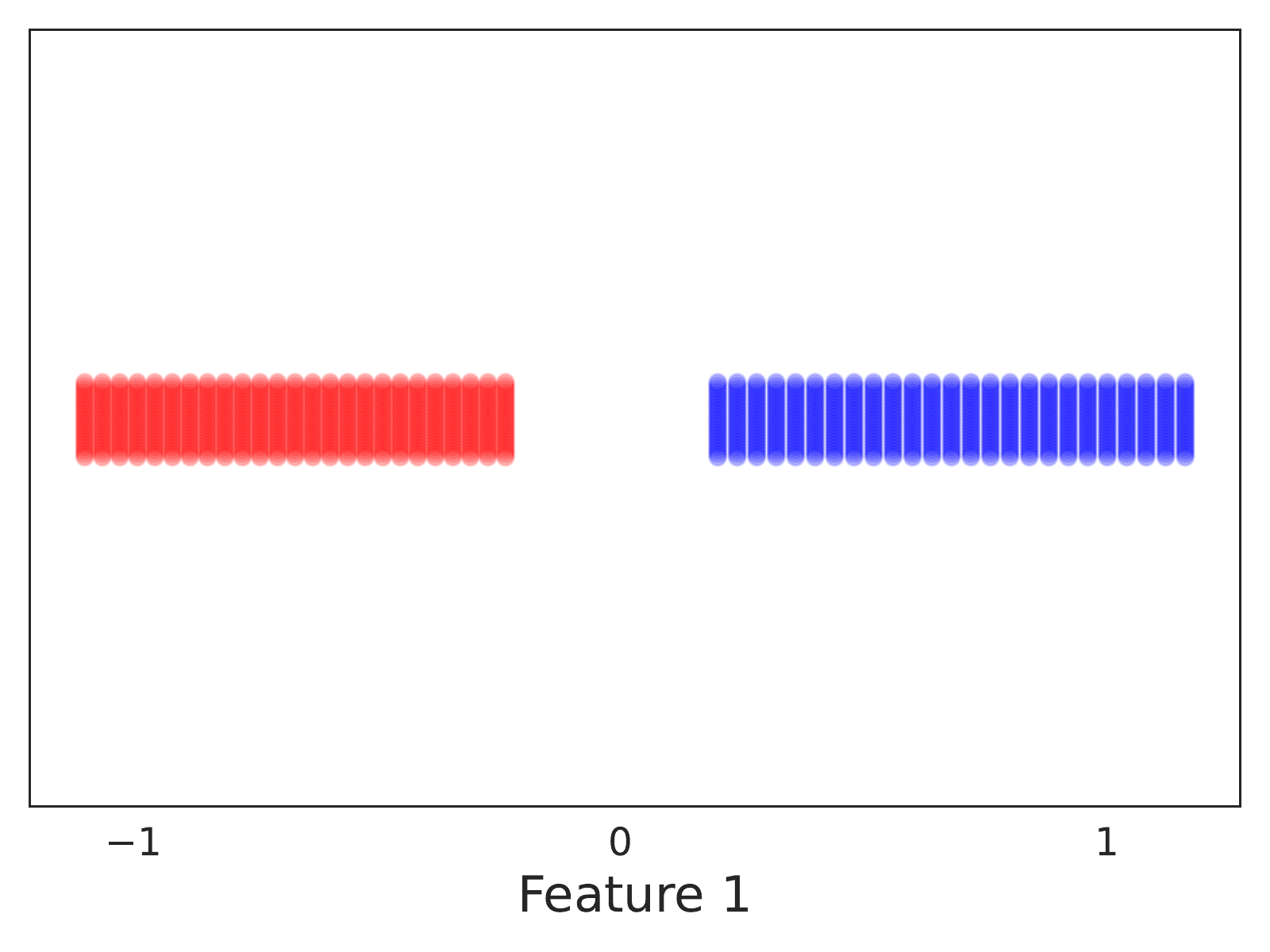}
    %\caption{Binary-Image}
    %\label{b-imagenet:eff_rank}
\end{subfigure}
\begin{subfigure}{0.3\columnwidth}
  \includegraphics[width=\columnwidth]{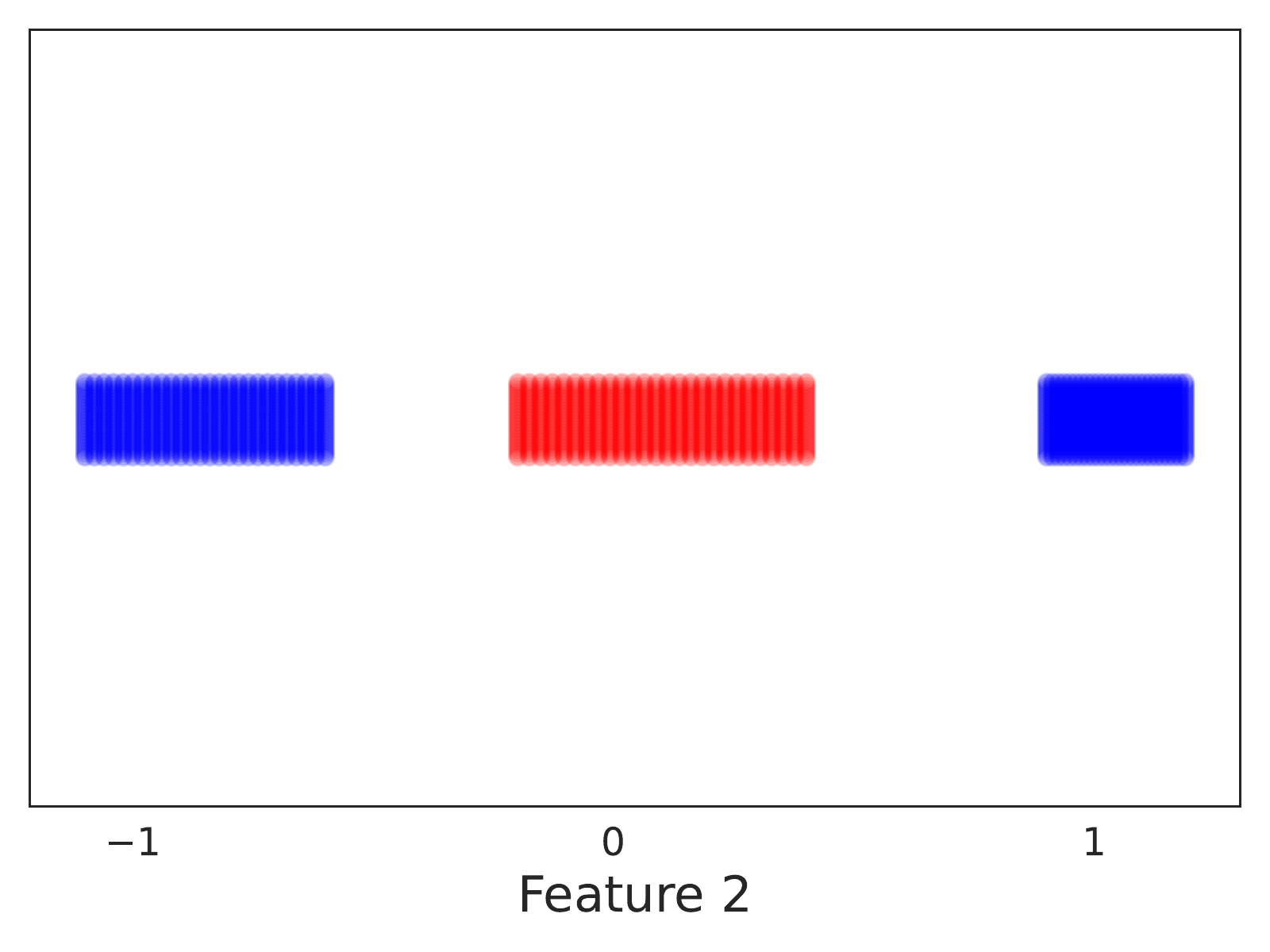}
    %\caption{Imagenette}
    %\label{imagenet:eff_rank}
\end{subfigure}
\begin{subfigure}{0.3\columnwidth}
  \includegraphics[width=\columnwidth]{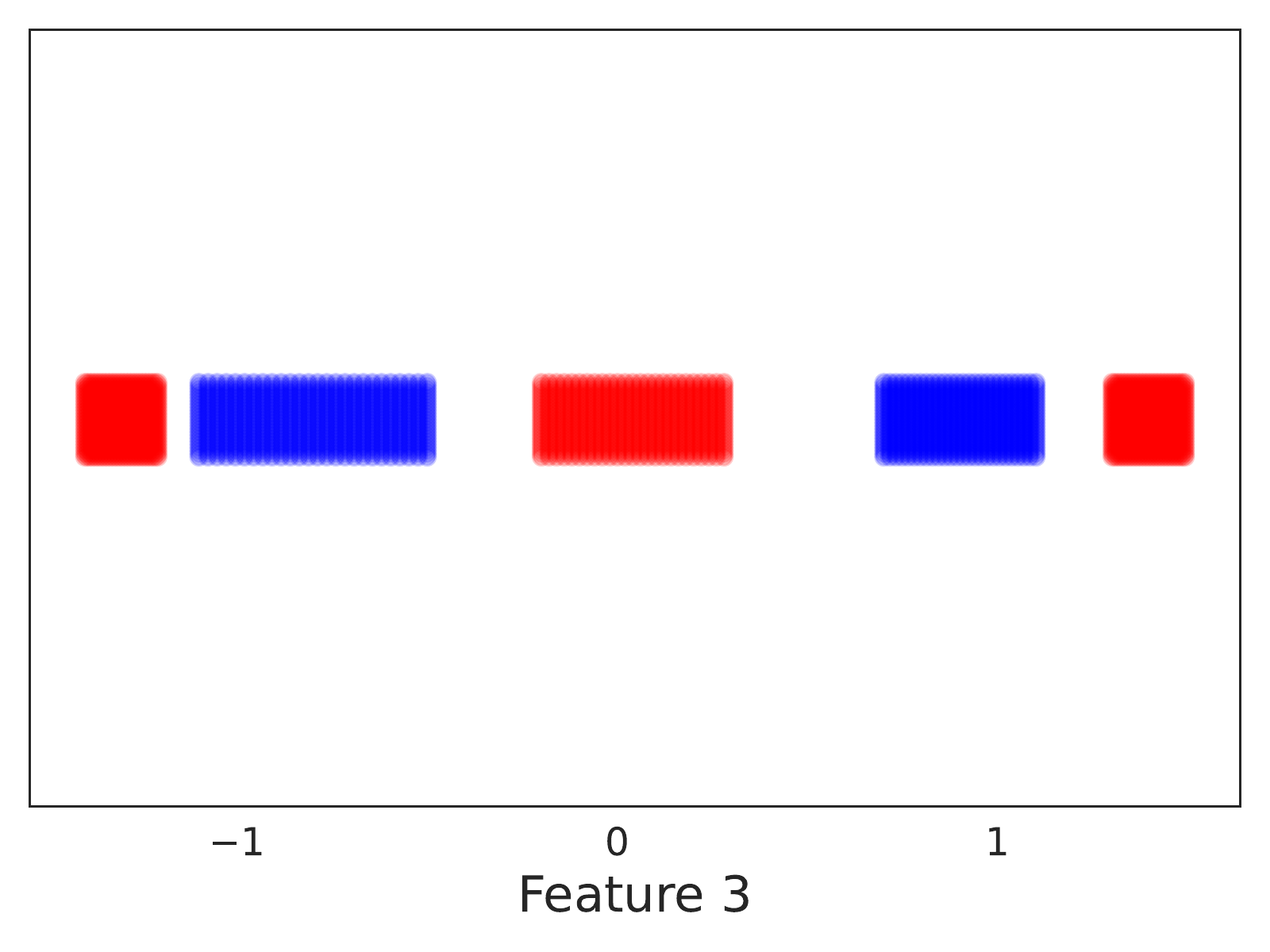}
    %\caption{Waterbirds-Landbirds}
    %\label{waterbirds:eff_rank}
\end{subfigure}
\caption{Illustration of an \ifm dataset. Given a class $\pm 1$ represented by {\color{blue} blue} and {\color{red} red} respectively, each coordinate value is drawn independently from the corresponding distribution. Shown above are the supports of distributions on three different coordinates for an illustrative \ifm dataset, for {\color{blue} positive} and {\color{red} negative} labels.}
\label{fig:synth}
\end{figure}
% \begin{figure}[t]
%   \centering
%   \includegraphics[width=0.5\textwidth]{plots/dataset.pdf}
%     \caption{Synthetic dataset}
%     \label{fig:synth}
% \end{figure}
% \praneeth{Depen: Please add the additional conditions on bounded density etc.}
We consider datasets in the independent features model (\ifm), where the joint distribution over $(x,y)$ satisfies $p(x,y) = r(y)\prod^d_{i=1}q_i(x_i|y)$, i.e, the features are distributed independently conditioned on the label $y$. Here $r(y)$ is a distribution over $\{-1,+1\}$ and  $q_i(x_i|y)$ denotes the conditional distribution of $i^{\textrm{th}}$-coordinate $x_i$ given $y$. \ifm is widely studied in literature, particularly in the context of naive-Bayes classifiers~\cite{lewis1998naive}. We make the following assumptions which posit that there are at least two features of differing complexity for classification: \emph{one} with a linear boundary and \emph{at least} one other with a non-linear boundary. See Figure~\ref{fig:synth} for an illustrative example.
\begin{itemize}[leftmargin=5mm]
    \item One of the coordinates (say, the $1^\textrm{st}$ coordinate WLOG) is separable by a linear decision boundary with margin $\gamma$ (see Figure \ref{fig:synth}), i.e, $\exists \gamma > 0$, such that $\gamma \in Supp(q_1(x_1|y=+1)) \subseteq [\gamma, \infty)$ and $-\gamma \in Supp(q_1(x_1|y=-1)) \subseteq (-\infty, -\gamma]$, where $Supp(\cdot)$ denotes the support of a distribution.
    %where $q_1(x_1|y)$ represents the conditional distribution of $x_1$ given $y$.
    \item None of the other coordinates is linearly separable. More precisely, for all the other coordinates $i \in [d] \setminus\{1\}$, $0 \in Supp(q_i(x_i|y=-1))$ and $\{-1,+1\} \subseteq Supp(q_i(x_i|y=+1))$. %This condition ensures that all other coordinates are not linearly separable.
    \item The dataset can be perfectly classified even without using the linear coordinate. This means, $\exists i \neq 1$, such that $q_i(x_i|y)$ has disjoint support for $y=+1$ and $y=-1$.
\end{itemize}
Though we assume axis aligned features, our results also hold for any rotation of the dataset. While our results hold in the general \ifm setting, in comparison, current results for SB e.g.,~\cite{shah2020pitfalls}, are obtained for \emph{very specialized} datasets within \ifm, and do not apply to \ifm in general.

% \cite{shah2020pitfalls} gave a synthetic dataset within the \ifm family for understanding simplicity bias. Roughly speaking, in their setting, each feature can perfectly predict the label but the decision boundaries for the respective features can be 'simple' (linear) or 'complex' (non-linear). They showed extreme simplicity bias for a certain stylized \ifm dataset (see Figure \ref{fig:synth}).

% Our datasets are significantly more general than those of \cite{shah2020pitfalls}, we now describe them in detail for the rich and lazy regime respectively.

%They give a notion of feature simplicity which is the complexity of decision boundary for each respective feature and
\subsection{Main result}\label{sec:rich}
Our main result states that,
for rich initialization (Section \ref{prelim:init}),
% we extend the results of \citep{shah2020pitfalls} to show
NNs demonstrate \ldsb for any \ifm dataset satisfying the above conditions, along with some technical conditions
% \footnote{This result relies on Theorem \ref{thm:CBrich}.
% The exact conditions for this to hold are 
stated in Theorem~\ref{thm:technical}. Its proof appears in Appendix \ref{app:rich}.
\begin{theorem}\label{thm:rich}
For any dataset in the \ifm model with bounded density and bounded support, satisfying the above conditions and
%stated in Section \ref{ifm:dataset}
$\gamma \geq 1$, if gradient flow for 1-hidden layer FCN under rich initialization in the infinite width limit with cross entropy loss converges and satisfies the technical conditions in Theorem~\ref{thm:technical} \footnote{Note that Theorem \ref{thm:technical} is a restatement of Theorem 5 of \citet{ChizatB20}, which we are using as a black box in our analysis}, it converges to $\nu^* = 0.5\delta_{\theta_1} + 0.5\delta_{\theta_2}$ on $\gS^{d+1}$, where $\theta_1 = (\frac{\gamma}{\sqrt{2(1+\gamma^2)}}\e_1,\frac{1}{\sqrt{2(1+\gamma^2)}},1/\sqrt{2}) ,\theta_2 = (-\frac{\gamma}{\sqrt{2(1+\gamma^2)}}\e_1,\frac{1}{\sqrt{2(1+\gamma^2)}},-1/\sqrt{2})$ and $\e_1\defeq [1,0, \cdots, 0]$ denotes first standard basis vector. This implies $f(\nu^*, Px^{(1)} + P_{\perp}x^{(2)}) = f(\nu^*, x_1^{(1)})$ $\forall \; (x^{(1)},y^{(1)}),(x^{(2)},y^{(2)}) \sim \gD$,
where $P$ represents the (rank-1) projection matrix on first coordinate.
% the condition in Theorem $\ref{thm:CBrich}$.
\end{theorem}
%First let us see how the above theorem implies extreme simplicity bias. 
% We further conjecture that the classifier in Theorem~\ref{thm:rich} is the \emph{unique} max-margin classifier.
Moreover, since at least one of the coordinates $\{2,\ldots,d\}$ has disjoint support for $q_i(x_i|y=+1)$ and $q_i(x_i|y=-1)$, $P_{\perp}(x)$ can still perfectly classify the given dataset, thereby implying \ldsb.

It is well known that the rich regime is more relevant for the practical performance of NNs since it allows for feature learning, while lazy regime does not~\citep{chizat2019lazy}.
% So, Theorem~\ref{thm:rich}, which handles the rich regime is more relevant to the practical performance of NNs.
Nevertheless, in the next section, we present theoretical evidence that \ldsb holds even in the lazy regime, by considering a much more specialized dataset within \ifm.
%\begin{remark}
%Note that the characterization of $\nu^*$ depends only on the support conditions as specified in the description of $\ifm$ dataset (condition 3 and 4 in Section \ref{ifm:dataset}). This means that our results still hold even if the features are dependent on each other or some of the coordinates do not have disjoint support.
%\end{remark}
%coordinates $2,3,\ldots,d$ of $\theta_1,\theta_2$ are all zero).

\subsection{Lazy regime}\label{sec:lazy}
%\praneeth{Give details of lazy regime. In particular initialization.}
% In this subsection, we analyze the behavior of the $\L_2$ max-margin NTK model (denoted by $f$), as defined in \Eqref{max-marg-ntk}. 
In this regime, we will work with the following dataset within the \ifm family:

For $y \in \{\pm 1\}$ we generate $(x,y) \in D$ as
\[ \x_1 = \gamma y \]
\[ \forall i \in {2,..,d}, \x_i = \left\{ \begin{array}{ccc}
\pm 1 & \mbox{for}
& y=1 \\ 0 & \mbox{for} & y=-1 \\
\end{array}\right.\]
Although the dataset above is a point mass dataset, it still exhibits an important characteristic in common with the rich regime dataset -- only one of the coordinates is linearly separable while others are not. For this dataset, we provide the characterization of max-margin NTK (as in \Eqref{max-marg-ntk}):

%we will show a similar but slightly weaker result that $f$ is not robust to adversarial perturbations 

%i.e. $f(x+\eta)$ has the opposite prediction as $f$ for $x \in D$ and $\|\eta\| = O(1)$. 

%We begin with preliminaries about max margin kernel SVM. the purposes of this subsection we will append each instance (in $\sR^d$) with a $1$ to account for the bias in the kernel SVM formulation. More precisely $f:\sR^{d+1} \rightarrow \sR$ is given by (via the representer theorem):
%\[f(x) = \sum^{2^{d-1}+1}_{i=0} a_iK(x,x^{(i)})\,,\]
%where $a_i$ are the non-negative Lagrange multipliers in the dual SVM optimization and $K$ is the NTK kernel in Theorem \ref{thm:lazyBM}. We now have,
\begin{theorem}\label{thm:lazyrandacc}
    For sufficiently small $\epsilon > 0$, there exists an absolute constant $N$ such that for all $d > N$ and $\gamma \in [7, \epsilon\sqrt{d})$, the $\gL_2$ max-margin classifier for joint training of both the layers of 1-hidden layer FCN in the NTK regime on the dataset $D$, i.e., any $f$ satisfying~\Eqref{max-marg-ntk} satisfies:
    \begin{align*}
        & \text{pred}(f(Px^{(1)} + P_{\perp}x^{(2)})) = \text{pred}(f(x^{(1)})) \\ & \qquad \forall \text{ } (x^{(1)},y^{(1)}), (x^{(2)},y^{(2)}) \in D
    \end{align*}
    where $P$ represents the projection matrix on the first coordinate and $\text{pred}(f(x))$ represents the predicted label by the model $f$ on $x$.
\end{theorem}
The above theorem shows that the prediction on a \emph{mixed} example $P x^{(1)} + P_\perp x^{(2)}$ is the same as that on $x^{(1)}$, thus establishing \ldsb. The proof for this theorem is provided in Appendix \ref{app:lazy}.

%In other words, the prediction depends only on the subspace given by $P$, establishing \ldsb. The proof for this theorem is provided in Appendix \ref{app:lazy}.
% the establishes SB in the lazy regime in terms of randomized accuracy. 
% Let us see why the above theorem exhibits extreme SB in terms of randomized accuracy, i.e, when we randomize the slab coordinates, the accuracy of the model remains 1 (this also implies that when we randomize the linear coordinate, the accuracy drops down to 0.5).   

%Idea of proof: write down prediction as a function of x explicitly
%1. for kernel SVM, 

\begin{table*}
\caption{Demonstration of \ldsb in the rich regime: This table presents $P_\perp$ and $P$ randomized accuracies (RA) as well as logit changes (LC) on the four datasets. These results confirm that projection of input $x$ onto the subspace spanned by $P$ essentially determines the model's prediction on $x$. $\uparrow$ (resp. $\downarrow$) indicates that \ldsb implies a large (resp. small) value.
% Evaluation of test accuracy, when projected on and orthogonal to a \praneeth{How many?} few top SVD vectors of the weight matrix. The last column shows the $\%$ change in logits when input is projected on a few top SVD vectors of the weight matrix. \praneeth{Depen: percentage change seems quite substantial. It suggests that the network does depend on these directions, although it is not making a difference to prediction?}\praneeth{May be put the relative difference in a sentence and leave it as an open question.}
}
\label{tab:proj_logits_rich}
\begin{center}
\setlength\tabcolsep{4.5pt}
\begin{tabular}{|M{2.2cm} M{1cm} M{1.7cm} M{1.7cm} M{1.7cm} M{1.9cm} M{2.2cm}|} 
 \hline
 Dataset & $\rank{P}$ & Acc($f$) & $P_\perp$-RA $(\uparrow)$ & $P$-RA $(\downarrow)$ & $P_\perp$-LC $(\downarrow)$ & $P$-LC $(\uparrow)$ \\ [0.5ex] 
 \hline\hline
%  Dataset & $\rank{P}$ & Acc($x$) & Acc($P$) & Acc($P_{\perp}$) & $100\frac{\| f(x) - f(Px)\|}{\|f(x)\|}$ & $P$ LC:\\ [0.5ex] 
%  \hline\hline
 b-Imagenette& $1$ & $93.05\pm 0.26$ & $89.94\pm0.22$ & $49.53\pm0.24$ & $28.57\pm0.26$ & $92.13\pm0.24$ \\ 
 Imagenette& $10$ & $79.52\pm0.13$ & $75.89\pm0.25$ & $9.33\pm0.01$ & $33.64\pm1.21$  & $106.29\pm0.53$ \\ 
 Waterbirds& $3$ & $91.88\pm0.1$ & $91.47\pm0.11$ & $62.51\pm0.07$ & $25.24\pm1.03$  & $102.35\pm0.19$ \\
 \mnistcifar& 1 & $99.69\pm0.0$ & $94.15\pm0.21$ & $55.2\pm0.13$ & $38.97\pm0.76$  & $101.98\pm0.31$ \\ 
 % Imagenet & 150 & $72.07\pm 0.05$ & $68.44\pm 0.01$ & $0.28\pm 0.0$ & $15.78\pm 0.05$ & $132.05\pm 0.06$ \\ [1ex] 
 \hline
\end{tabular}
\end{center}
\end{table*}

\subsection{Empirical verification}
In this section, we will present empirical results demonstrating \ldsb on $3$ real datasets: Imagenette \citep{imagenette}, a binary version of Imagenette (b-Imagenette) and waterbirds-landbirds \citep{Sagawa*2020Distributionally}
% and Imagenet \citep{deng2009imagenet}
as well as one designed dataset \mnistcifar~\citep{shah2020pitfalls}. More details about the datasets can be found in Appendix~\ref{app:exp-setting}.

\subsubsection{Experimental setup}
We take Imagenet pretrained Resnet-50 models, with $2048$ features, for feature extraction and train a $1$-hidden layer fully connected network, with ReLU nonlinearity, and $100$ hidden units, for classification on each of these datasets. During the finetuning process, we freeze the backbone Resnet-50 model and train only the $1$-hidden layer head (more details in Appendix~\ref{app:exp-setting}) . 
% More details about the optimizers and hyperparameter tuning can be found in Appendix~\ref{app:exp-setting}.
% pretrained representations of these datasets. A 1-hidden layer FCN having 100 neurons in the hidden layer and ReLU activation is trained using
% We use SGD with momentum optimizer and perform hyperparameter tuning over learning rate, weight decay and batch size using a validation set. All reported results are averaged over 3 runs with hyperparameter selection using best validation accuracy. A more detailed description of the training and evaluation procedure is presented in the Appendix~\ref{app:exp-setting}.

% \begin{itemize}
%     \item Datasets
%     \item Pretrained representations
%     \item Width of $1$-hidden layer model
%     \item Learning rate, optimizers and other hyperparameters as well as hyperparameter tuning setup.
% \end{itemize}

\textbf{Demonstrating \ldsb}:
% \praneeth{Use randomized accuracy instead of project in and project out?}
Given a model $f(\cdot)$, we establish its low dimensional SB by identifying a small dimensional subspace, identified by its projection matrix $P$, such that if we \emph{mix} inputs $x_1$ and $x_2$ as $P x_1 + P_\perp x_2$, the model's output on the mixed input $\xtilde \defeq Px_1 + P_\perp x_2$,  $f(\xtilde)$ is always \emph{close} to the model's output on $x_1$ i.e., $f(x_1)$. We measure \emph{closeness} in four metrics: (1) $P_\perp$-randomized accuracy ($P_\perp$-RA):  accuracy on the dataset $(P x_1+ P_\perp x_2, y_1)$ where $(x_1,y_1)$ and $(x_2, y_2)$ are sampled iid from the dataset, (2) $P$-randomized accuracy ($P$-RA): accuracy on the dataset $(P x_1+ P_\perp x_2, y_2)$, (3) $P_\perp$ logit change ($P_\perp$-LC): relative change wrt logits of $x_1$ i.e., $\norm{f(\xtilde)-f(x_1)}/\norm{f(x_1)}$, and (4)$P$ logit change ($P$-LC): relative change wrt logits of $x_2$ i.e., $\norm{f(\xtilde)-f(x_2)}/\norm{f(x_2)}$. Moreover, we will also show that a subsequent model trained on $(P_{\perp}x,y)$ achieves significantly high accuracy on these datasets.
% projecting inputs onto this subspace attains similar accuracy while projecting inputs orthogonal to this subspace drastically degrades accuracy. More concretely, given a model $f$ with inputs from $\R^d$ and a subspace $S \subseteq \R^d$, we compute (i) \projin accuracy (Acc($P$)) i.e., accuracy of $x \rightarrow f(P x)$, (ii) \projout accuracy (Acc($P_\perp$)) i.e., accuracy of $x \rightarrow f(P_\perp x)$ and (iii) relative change of logits with projection i.e., $\frac{\norm{f(x)-f(Px)}}{\norm{f(x)}}$.
% We will consider results on four real-world datasets:
% \begin{enumerate}
%     \item Waterbirds-Landbirds
%     \item Imagenette
%     \item Binary-Imagenette
%     \item CIFAR-MNIST
% \end{enumerate}

% \subsection{Low rank SB on real datasets}
\begin{table*}
\caption{Demonstration of \ldsb in the lazy regime: This table presents $P_\perp$ and $P$ randomized accuracies as well as logit changes on the four datasets. These results confirm that the projection of input $x$ onto the subspace spanned by $P$ essentially determines the model's prediction on $x$.
% Evaluation of test accuracy, when projected on and orthogonal to a \praneeth{How many?} few top SVD vectors of the weight matrix. The last column shows the $\%$ change in logits when input is projected on a few top SVD vectors of the weight matrix. \praneeth{Depen: percentage change seems quite substantial. It suggests that the network does depend on these directions, although it is not making a difference to prediction?}\praneeth{May be put the relative difference in a sentence and leave it as an open question.}
}
\label{tab:proj_logits_lazy}
\begin{center}
\setlength\tabcolsep{4.5pt}
\begin{tabular}{|M{2.2cm} M{1cm} M{1.7cm} M{1.7cm} M{1.7cm} M{1.7cm} M{1.9cm}|}
 \hline
 Dataset & $\rank{P}$ & Acc($f$) & $P_\perp$-RA $(\uparrow)$ & $P$-RA $(\downarrow)$ & $P_\perp$-LC $(\downarrow)$ & $P$-LC $(\uparrow)$ \\ [0.5ex] 
 \hline\hline
%  Dataset & $\rank{P}$ & Acc($x$) & Acc($P$) & Acc($P_{\perp}$) & $100\frac{\| f(x) - f(Px)\|}{\|f(x)\|}$ & $P$ LC:\\ [0.5ex] 
%  \hline\hline
b-Imagenette& $1$ & $92.75\pm 0.06$ & $90.07\pm0.34$ & $52.09\pm1.34$ & $36.94\pm1.01$ & $138.41\pm1.62$ \\ 
 Imagenette& $15$ & $79.97\pm0.44$ & $68.25\pm1.18$ & $11.92\pm0.82$ & $55.99\pm3.86$  & $133.86\pm5.42$ \\ 
 Waterbirds& $6$ & $90.46\pm0.07$ & $89.67\pm0.42$ & $62.44\pm4.48$ & $36.89\pm5.18$  & $105.41\pm7.06$ \\
 \mnistcifar& $2$ & $99.74\pm0.0$ & $99.45\pm0.17$ & $49.83\pm0.67$ & $24.9\pm0.61$  & $141.12\pm1.86$ \\
 % Imagenet& $200$ & $72.6\pm0.00$ & $67.97 \pm0.14$ & $0.27 \pm0.02$ & $32.74 \pm0.02$ & $132.47 \pm0.04$ \\ [1ex] 
%  b-Imagenette& $1$ & $92.58\pm 0.18$ & $90.07\pm0.49$ & $51.32\pm0.26$ & $35.35\pm0.24$ & $136.08\pm0.7$ \\ 
%  Imagenette& $15$ & $80.1\pm0.38$ & $68.7\pm0.2$ & $12.36\pm0.8$ & $53.41\pm0.72$  & $131.38\pm0.28$ \\ 
%  Waterbirds& $4$ & $90.53\pm0.18$ & $90.48\pm0.09$ & $57.22\pm0.19$ & $38.29\pm0.73$  & $112.91\pm0.2$ \\
%  \mnistcifar& $2$ & $99.8\pm0.0$ & $99.56\pm0.1$ & $49.08\pm0.92$ & $29.7\pm2.03$  & $142.14\pm1.38$ \\ [1ex] 
 \hline
\end{tabular}
\end{center}
\end{table*}
\begin{table*}
\caption{Mistake diversity and class conditioned logit correlation of models trained independently ($\mdiv{f,f_{\textrm{ind}}}$ and $\cclc{f,f_{\textrm{ind}}}$ resp.) vs trained sequentially after projecting out important features of the first model ($\mdiv{f,f_{\textrm{proj}}}$ and $\cclc{f,f_{\textrm{proj}}}$ resp.). The results demonstrate that $f$ and $f_{\textrm{proj}}$ are more diverse compared to $f$ and $f_\textrm{ind}$.}
\label{tab:err_div_cclc}
\begin{center}
\begin{tabular}{|M{2.2cm} M{2.2cm} M{2.2cm} M{2.2cm} M{2.2cm}|} 
 \hline
 Dataset & $\textrm{Mist-Div}$ $\left(f,f_\textrm{ind}\right)$ $(\uparrow)$ & $\textrm{Mist-Div}$ $\left(f,f_\textrm{proj}\right)$   $(\uparrow)$ & $\textrm{CC-LogitCorr}$ $\left(f,f_\textrm{ind}\right) (\downarrow)$ & $\textrm{CC-LogitCorr}$ $\left(f,f_\textrm{proj}\right) (\downarrow)$ \\ [0.5ex] 
 \hline\hline
 B-Imagenette & $3.87\pm 1.54$ & $21.15\pm1.57$ & $99.88\pm0.01$ & $90.86\pm1.08$ \\ 
 Imagenette & $6.6\pm0.46$ & $11.44\pm0.65$ & $99.31\pm0.12$ & $91\pm0.59$ \\ 
 Waterbirds & $2.9\pm0.52$ & $14.53\pm0.48$ & $99.66\pm0.04$ & $93.81\pm0.48$ \\
 MNIST-CIFAR & $0.0\pm0.0$ & $5.56\pm7.89$ & $99.76\pm0.17$ & $78.74\pm2.28$ \\ [1ex]
 %Imagenet & $6.97\pm 0.06$ & $12.31\pm 0.16$ & $99.5\pm 0.0$ & $92.52\pm 0.01$\\[1ex] 
 \hline
\end{tabular}
\end{center}
\end{table*}
As described in Sections~\ref{sec:rich} and~\ref{sec:lazy}, the training of $1$-hidden layer neural networks might follow different trajectories depending on the scale of initialization. So, the subspace projection matrix $P$ will be obtained in different ways for rich vs lazy regimes. For rich regime, we will empirically show that the first layer weights have a low rank structure as per Theorem~\ref{thm:rich} while for lazy regime, we will show that though first layer weights do not exhibit low rank structure, the model still has low dimensional dependence on the input as per Theorem~\ref{thm:lazyrandacc}.
\subsubsection{Rich regime}
Theorem~\ref{thm:rich} suggests that asymptotically, the first layer weight matrix will be low rank. However, since we train only for a finite amount of time, the weight matrix will only be approximately low rank. To quantify this, we use the notion of effective rank~\cite{roy2007effective} to measure the rank of the first layer weight matrix.
\begin{definition}\label{def:effrank}
Given a matrix $M$, its effective rank is defined as:
% \begin{align*}
    $\textrm{Eff-rank}(M) = e^{-\sum_i \overline{\sigma_i(M)^2} \log \overline{\sigma_i(M)^2}}$
% \end{align*}
where $\sigma_i(M)$ denotes the $i^{\textrm{th}}$ singular value of $M$ and $\overline{\sigma_i(M)^2} \defeq \frac{\sigma_i(M)^2}{\sum_i \sigma_i(M)^2}$.
\end{definition}

\begin{figure}%
\centering
\begin{subfigure}{0.49\columnwidth}
  \centering
  \includegraphics[width=\columnwidth]{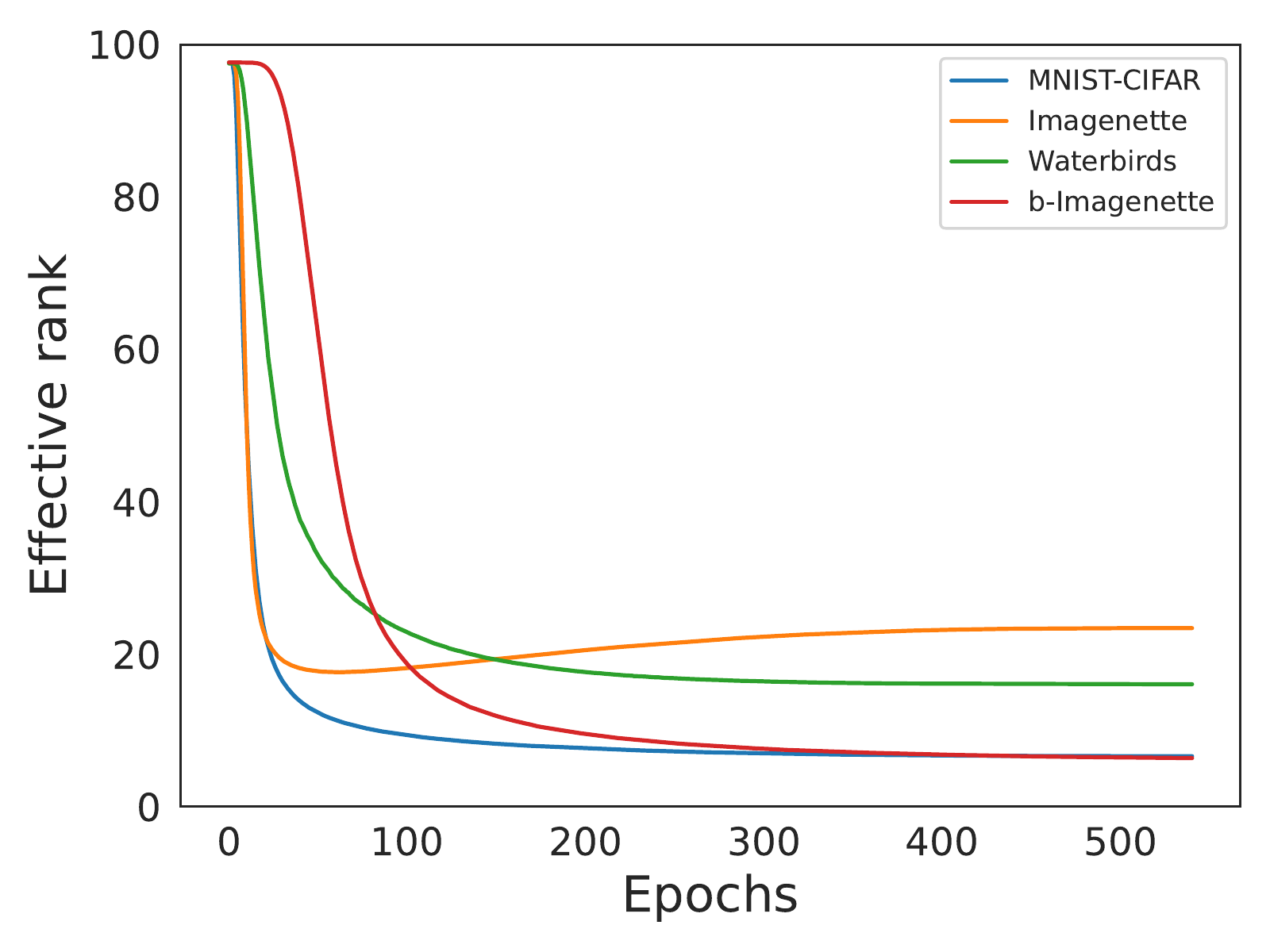}
    \caption{Rich regime}
    \label{fig:rich_eff_rank}
\end{subfigure}
\begin{subfigure}{0.49\columnwidth}
  \centering
  \includegraphics[width=\columnwidth]{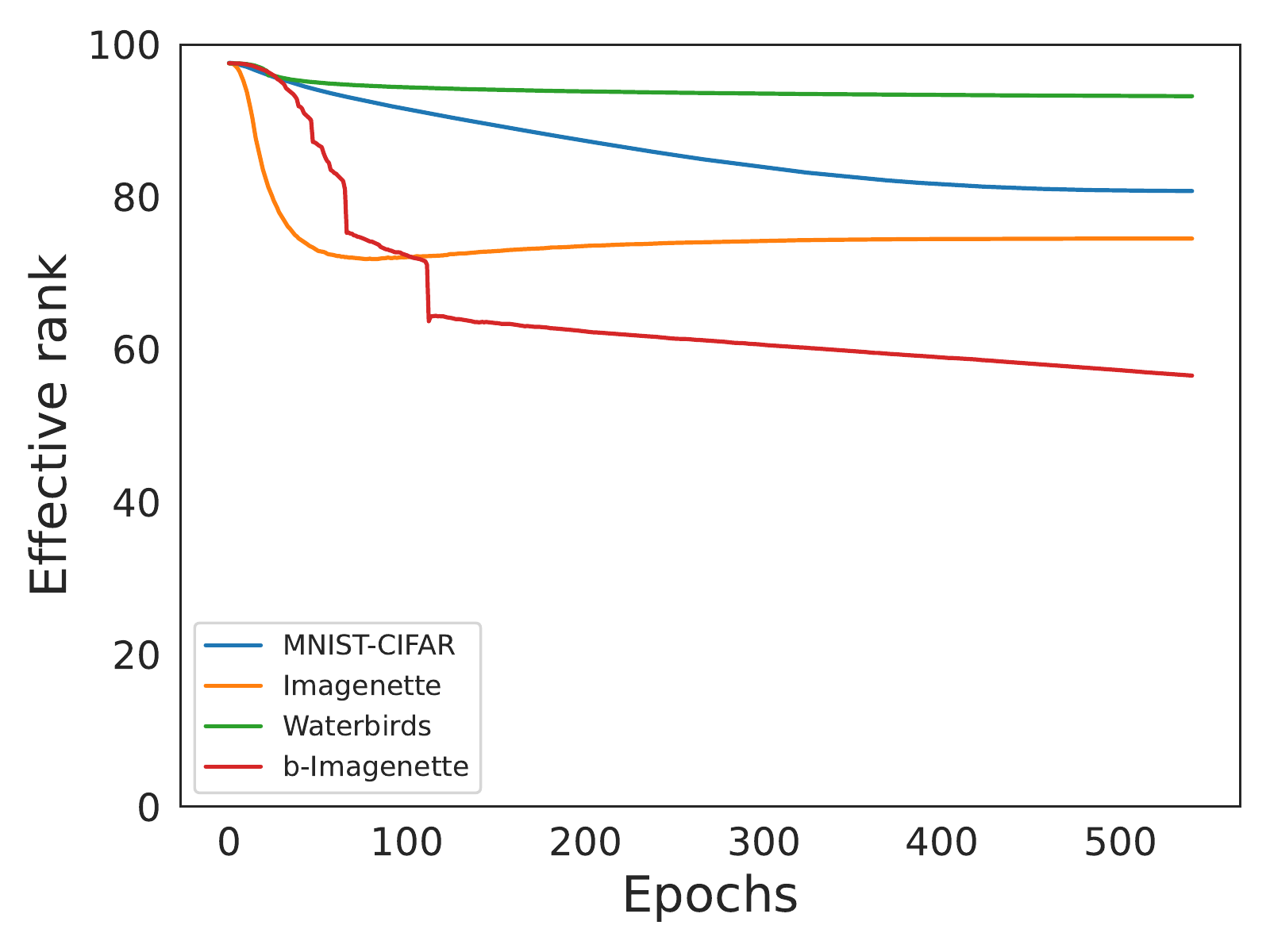}
    \caption{Lazy regime}
    \label{fig:lazy_eff_rank}
\end{subfigure}
\caption{Evolution of effective rank of first layer weight matrices in rich and lazy regimes.}
\label{fig:eff_rank}
\end{figure}

One way to interpret the effective rank is that it is the exponential of von-Neumann entropy~\cite{petz2001entropy} of the matrix $\frac{M M^\top}{\trace{M M^\top}}$, where $\trace{\cdot}$ denotes the trace of a matrix. For illustration, the effective rank of a projection matrix onto $k$ dimensions equals $k$.

Figure~\ref{fig:rich_eff_rank} shows the evolution of the effective rank through training on the four datasets. We observe that the effective rank of the weight matrix decreases drastically towards the end of training. 
To confirm that this indeed leads to \ldsb, we set $P$ to be the subspace spanned by the top singular directions of the first layer weight matrix and compute $P$ and $P_\perp$ randomized accuracies as well as the relative logit change. The results, presented in Table~\ref{tab:proj_logits_rich} confirm \ldsb in the rich regime on these datasets. Moreover, in Appendix \ref{app:additional-exps}, in Table \ref{tab:f_proj_acc}, we show that an independent model trained on $(P_{\perp} x, y)$ achieves significantly high accuracy.
% To confirm that this indeed leads to low dimensional SB, we measure the \projin and \projout accuracies (and logits) of the network with the subspace spanned by the top \praneeth{How many?} singular directions of first layer weight matrix.
% and compute the accuracy of the model when we replace the inputs $x$ by $\mbP X$ as well as $\mbP_\perp x$ where $\mbP_\perp$ is the projection onto the orthogonal space of $\mbP$.
% We will represent the standard accuracy by $\text{Acc}(x)$, the project-in accuracy by $\text{Acc}(\mathbb{P}(x))$ and projected out accuracy by $\text{Acc}(\mathbb{P}_{\perp}(x))$. Similarly, we will represent the projected onto logits by $f(\mathbb{P}(x))$.
% Table~\ref{tab:proj_logits_rich} presents these results. As can be seen, \projin accuracy is almost the same as the original accuracy while \projout accuracy is close to random guessing \praneeth{Depen: verify}.

% For a 100-dimensional hidden layer FCN, the evolution of effective rank with rich initialization for various datasets is shown in Fig \ref{rich:eff_rank}. SImilarly, for lazy regime, the plots are shown in Fig \ref{lazy:eff_rank}.

\subsubsection{Lazy regime}
% \praneeth{TODO: Depen}
For the lazy regime, it turns out that the rank of first layer weight matrix remains high throughout training, as shown in Figure~\ref{fig:lazy_eff_rank}. However, we are able to find a low dimensional projection matrix $P$ satisfying the conditions of LD-SB (as stated in Def \ref{def:ldsb}) as the solution to an optimization problem. More concretely, given a pretrained model $f$ and a rank $r$, we obtain a \emph{projection matrix} $P$ solving:
\begin{align}
    \min_{P} \frac{1}{n} \sum_{i=1}^n \left(\ce{f(Px^{(i)}),y^{(i)}} + \lambda \ce{f(P^{\perp}x^{(i)}), \U[L]}\right) \label{eqn:opt-P}
\end{align}
where $\U[L]$ represents a uniform distribution over all the $L$ labels, $(x^{(1)}, y^{(1)}),\cdots,(x^{(n)}, y^{(n)})$ are training examples and $\ce{\cdot,\cdot}$ is the cross entropy loss. We reiterate that the optimization is only over $P$, while the model parameters $f$ are unchanged. In words, the above function ensures that the neural network produces correct predictions along $P$ and uninformative predictions along $P_{\perp}$. Table~\ref{tab:proj_logits_lazy} presents the results for $P_\perp$ and $P$-RA as well as LC. As can be seen, even in this case, we are able to find small rank projection matrices demonstrating \ldsb. Similar to the rich regime, in Appendix \ref{app:additional-exps}, in Table \ref{tab:f_proj_acc_lazy}, we show that an independent model trained on $(P_{\perp}x, y)$ in the lazy regime achieves significantly high accuracy.

\section{Training diverse classifiers using \emph{\diverse}}\label{sec:diverse}

\begin{figure*}[t]
\centering
\begin{subfigure}{.4\columnwidth}
  \centering
  \includegraphics[width=\columnwidth]{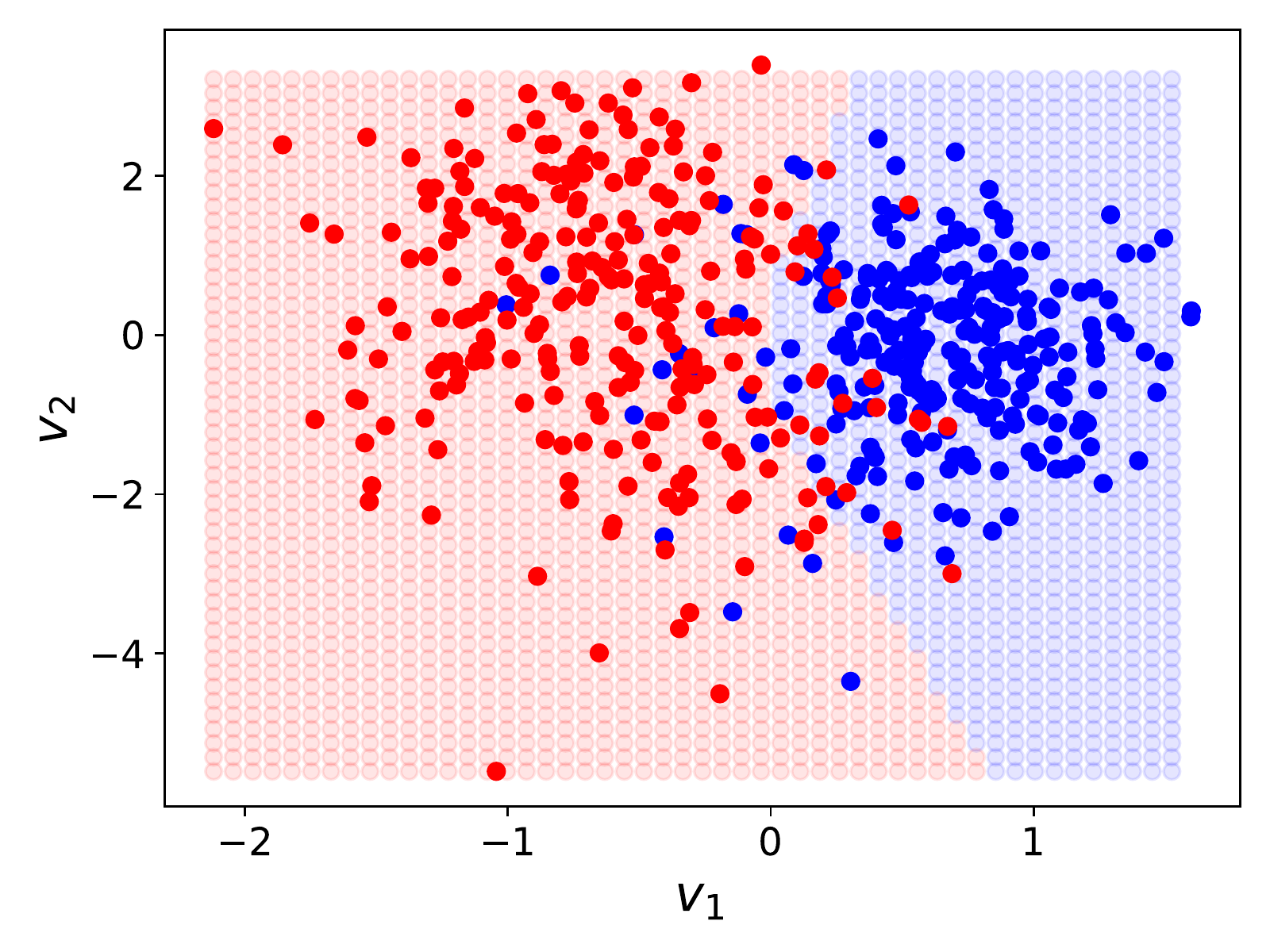}%
    \caption{b-Imagenette ($f$)}
    \label{b-imagenet:eff_rank}%
\end{subfigure}
\begin{subfigure}{.4\columnwidth}
  \centering
  \includegraphics[width=\columnwidth]{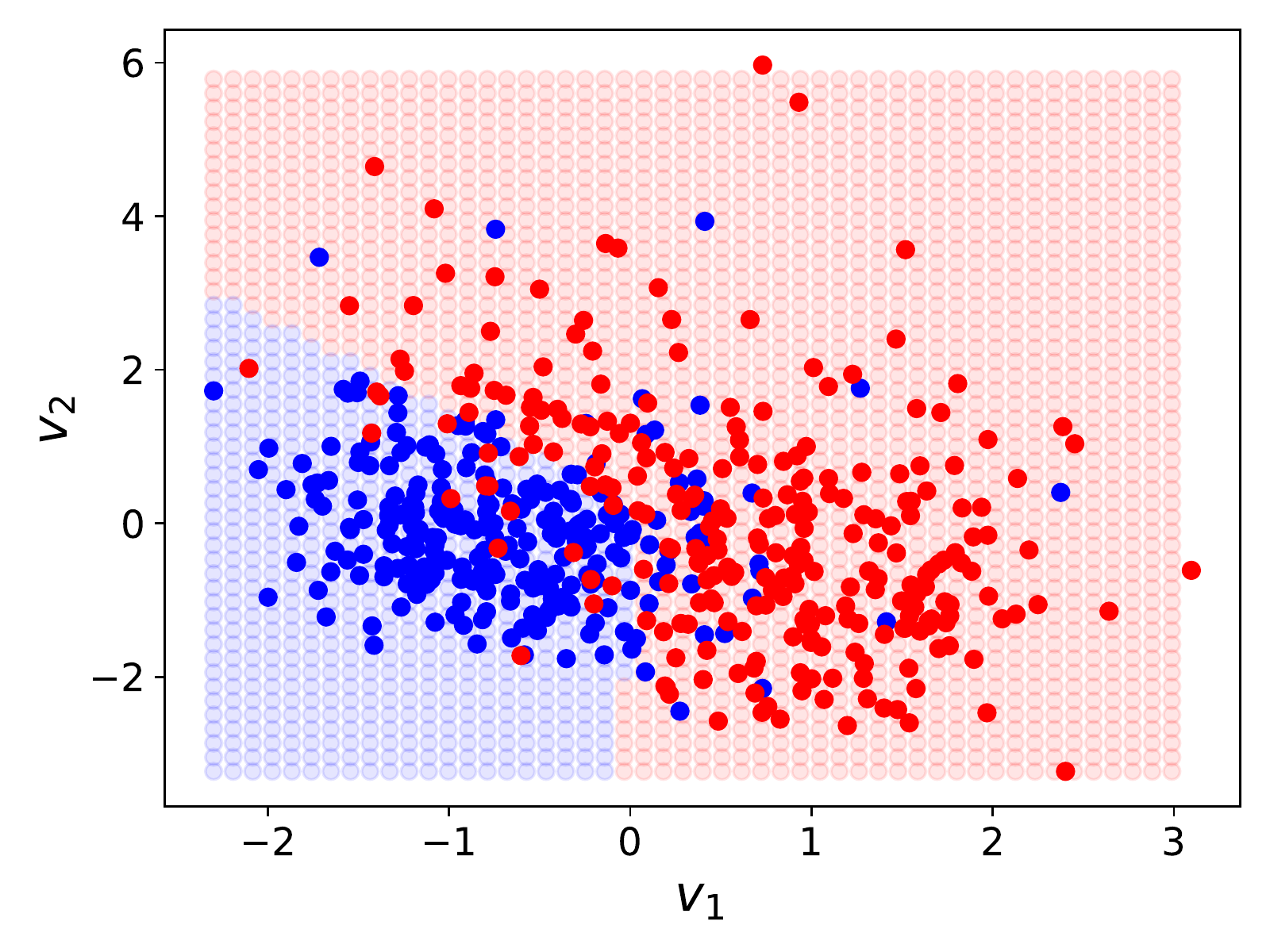}
    \caption{b-Imagenette ($f_\textrm{proj}$)}%
    \label{b-imagenet:eff_rank}%
\end{subfigure}
\begin{subfigure}{.4\columnwidth}
  \centering
  \includegraphics[width=\columnwidth]{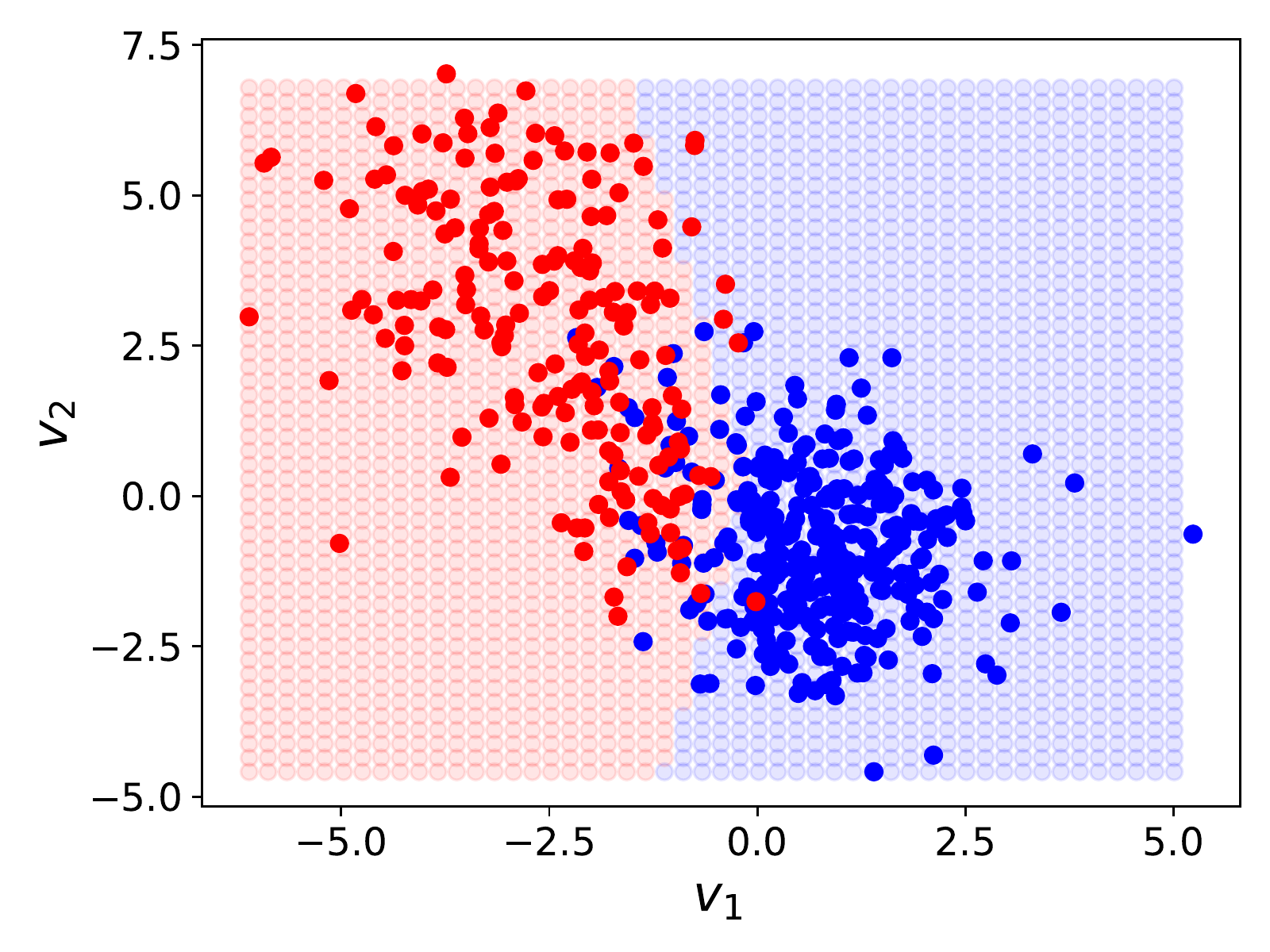}
    \caption{Waterbirds ($f$)}%
    \label{waterbirds:eff_rank}%
\end{subfigure}
\begin{subfigure}{.4\columnwidth}
  \centering
  \includegraphics[width=\columnwidth]{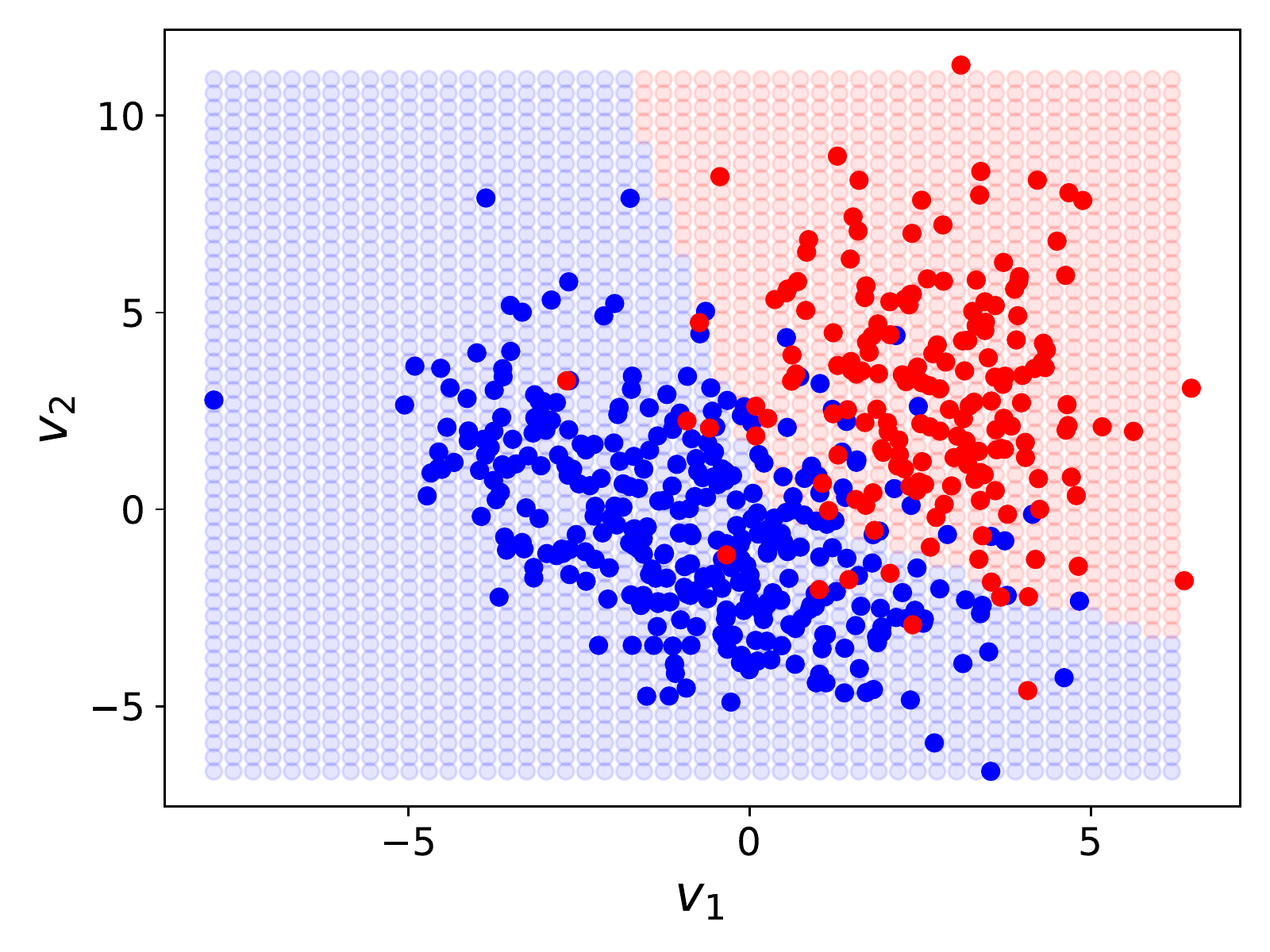}
    \caption{Waterbirds ($f_\textrm{proj}$)}%
    \label{waterbirds:eff_rank}%
\end{subfigure}
% \begin{subfigure}{.5\textwidth}
%   \centering
%   \includegraphics[width=\textwidth]{plots/mnist-cifar/model1/decision_boundary_with_points.pdf}
%     \caption{MNIST-CIFAR (model1)}
%     \label{mnist-cifar:eff_rank}
% \end{subfigure}
% \begin{subfigure}{.5\textwidth}
%   \centering
%   \includegraphics[width=\textwidth]{plots/mnist-cifar/decision_boundary_with_points.pdf}
%     \caption{MNIST-CIFAR (model2)}
%     \label{mnist-cifar:eff_rank}
% \end{subfigure}
\caption{Decision boundaries for $f$ and $f_\textrm{proj}$ for B-Imagenette and Waterbirds datasets, visualized in the top $2$ singular directions of the first layer weight matrix. The decision boundary of $f_{\textrm{proj}}$ is more non-linear compared to that of $f$.}
\label{fig:dec_bound}
\end{figure*}

\begin{figure*}
\centering
\begin{subfigure}{.5\columnwidth}
  \centering
  \includegraphics[width=\columnwidth]{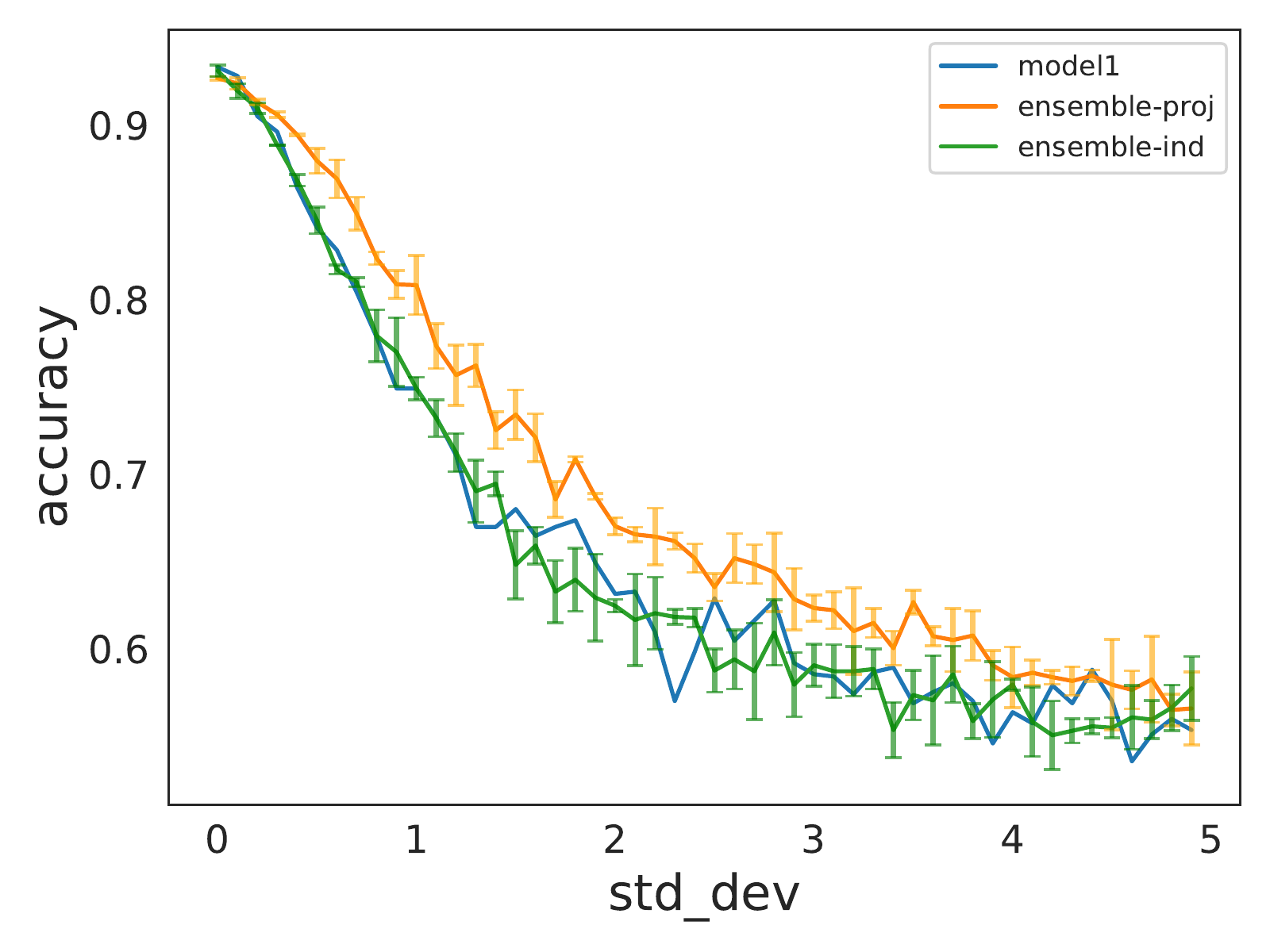}
    \caption{Binary-Imagenette}
    \label{b-imagenet:eff_rank}
\end{subfigure}
\begin{subfigure}{.5\columnwidth}
  \centering
  \includegraphics[width=\columnwidth]{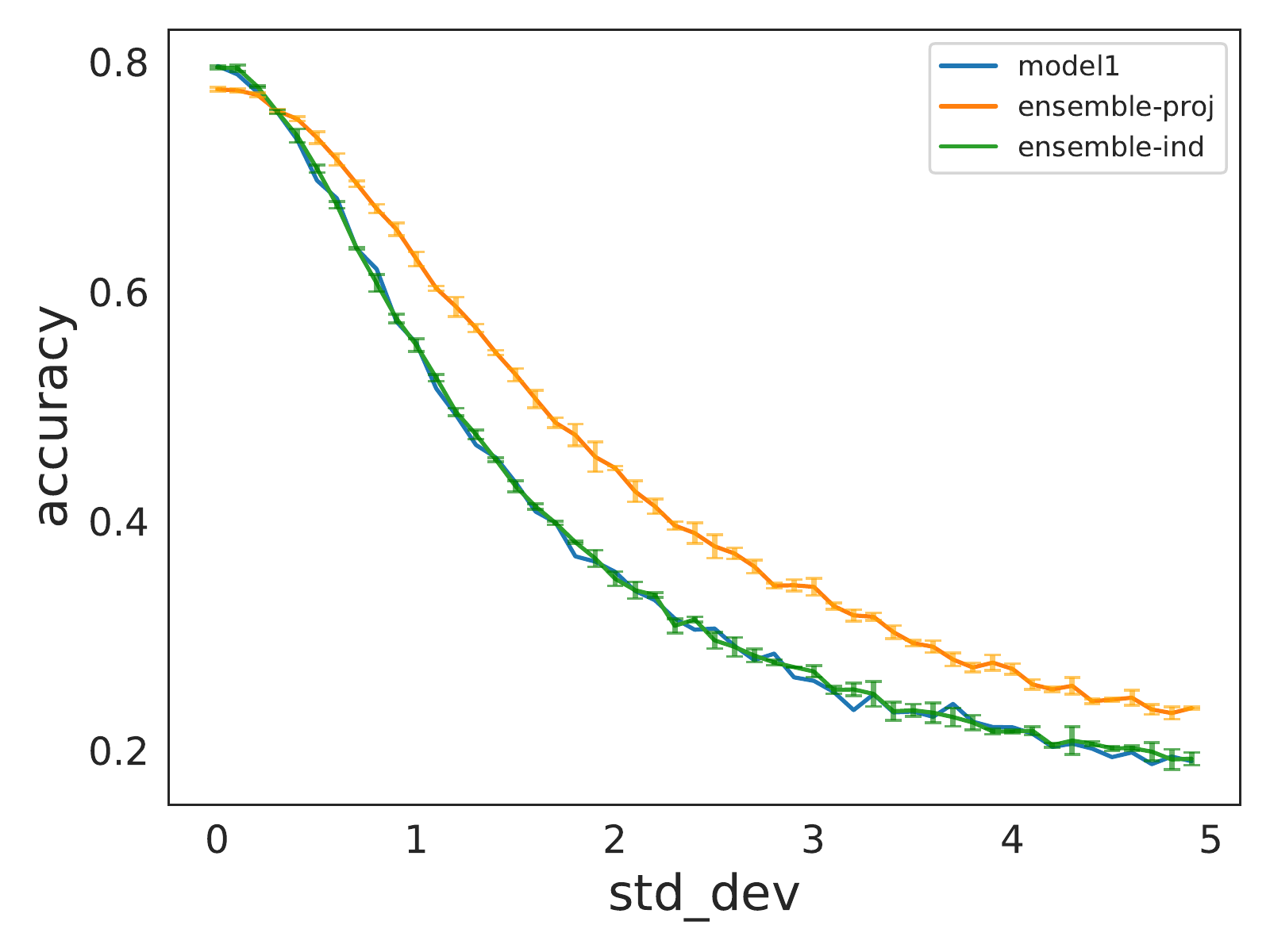}
    \caption{Imagenette}
    \label{imagenet:eff_rank}
\end{subfigure}
\begin{subfigure}{.5\columnwidth}
  \centering
  \includegraphics[width=\columnwidth]{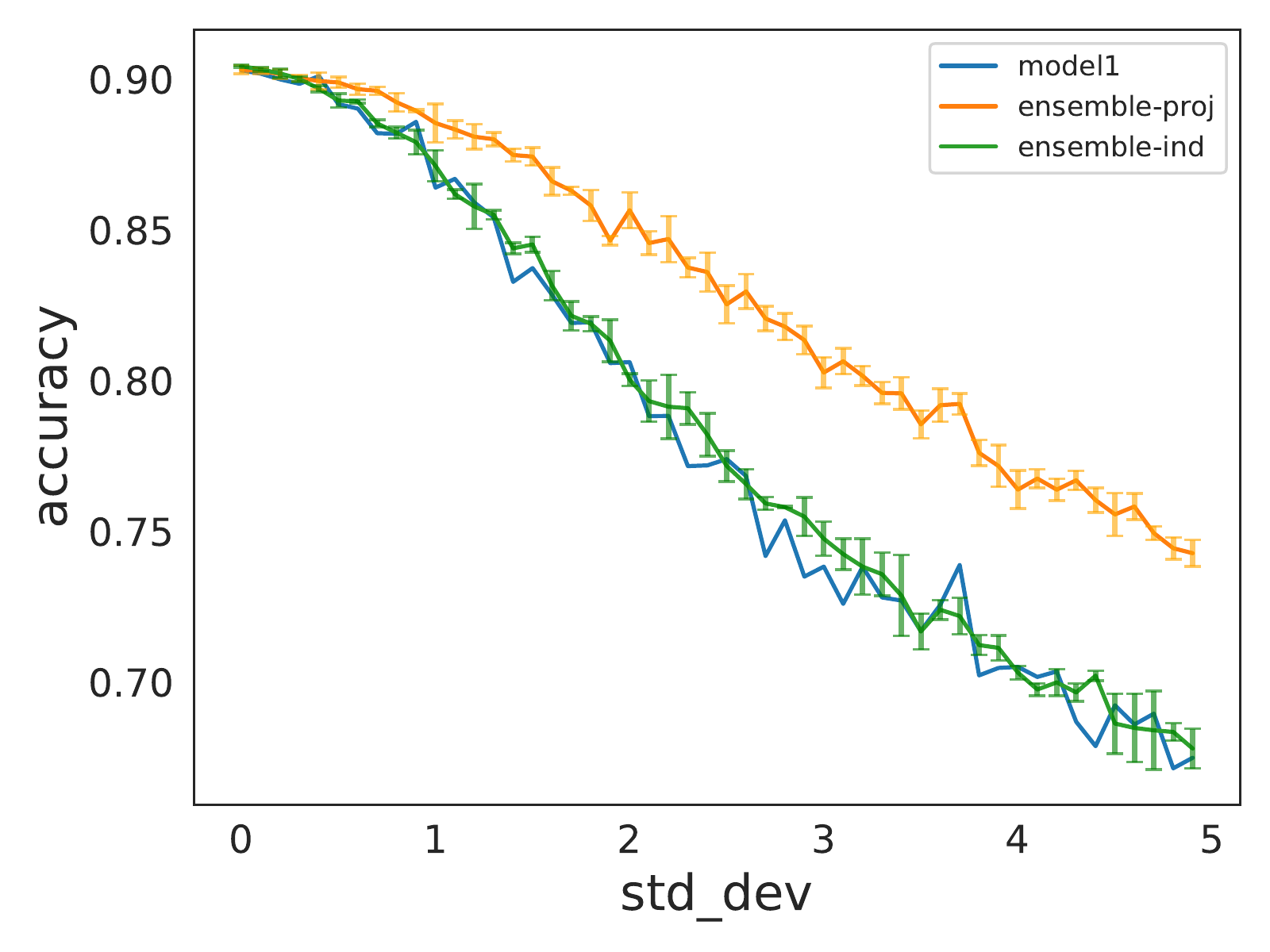}
    \caption{Waterbirds-Landbirds}
    \label{waterbirds:eff_rank}
\end{subfigure}
% \begin{subfigure}{.5\columnwidth}
%   \centering
%   \includegraphics[width=\columnwidth]{plots/imagenet/gauss_robust_test_overall_3_plots_stddev.pdf}
%     \caption{Imagenet}
%     \label{imagenet:eff_rank}
% \end{subfigure}
% \begin{subfigure}{.5\textwidth}
%   \centering
%   \includegraphics[width=\textwidth]{plots/mnist-cifar/gauss_robust_test_overall_3_plots.pdf}
%     \caption{MNIST-CIFAR}
%     \label{mnist-cifar:eff_rank}
% \end{subfigure}
\caption{Variation of test accuracy vs standard deviation of Gaussian noise added to the pretrained representations of the dataset. Model 1 (i.e., $f$) is kept fixed, and values for both the ensembles are averaged across 3 runs. Standard deviation is shown by the error bars.}
\label{fig:rich_gauss_robust}
\end{figure*}
Motivated by our results on low dimensional SB, in this section, we present a natural way to train diverse models, so that an ensemble of such models could mitigate SB. More concretely, given an initial model $f$ trained with rich regime initialization, we first compute the low dimensional projection matrix $P$ using the top few PCA components of the first layer weights.
% , if $f$ was initialized in the rich regime, or (2) by solving~\eqref{eqn:opt-P} if $f$ was initialized in the lazy regime.}
% with a low dimensional projection $P$ that captures its input dependence,

We then train another model $f_{\textrm{proj}}$ by projecting the input through $P_\perp$ i.e., instead of using dataset $(x^{(i)}, y^{(i)})$ for training, we use $(P_\perp x^{(i)}, y^{(i)})$ for training the second model (denoted by $f_{\textrm{proj}}$). We refer to this training procedure as $\diverse$ for \emph{orthogonal projection}. First, we show that this method provably learns different set of features for any dataset within \ifm in the rich regime. Then, we provide two  natural diversity metrics and demonstrate that $\diverse$ leads to diverse models on practical datasets. Finally, we provide a natural way of ensembling these models and demonstrate than on real-world datasets such ensembles can be significantly more robust than the baseline model. 

\textbf{Theoretical proof for \ifm}: First, we theoretically establish that $f$ and $f_{\textrm{proj}}$ obtained via $\diverse$ rely on different features for any dataset within \ifm. Consequently, by the definition of \ifm, $f$ and $f_{\textrm{proj}}$ have independent logits conditioned on the class. Its proof appears in Appendix \ref{app:prop-diverse}.
\begin{proposition} \label{prop:diverse-ifm}
Consider any \ifm dataset as described in Section~\ref{ifm:dataset}.
Let $f$ be the model described in Theorem~\ref{thm:CBrich} and $f_\textrm{proj}$ be the second model obtained by $\diverse$. Then, the outputs $f$ and $f_\textrm{proj}$ on $x$ i.e., $f(x)$ and $f_\textrm{proj}(x)$ depend only on $x_1$ and $\set{x_2, \cdots, x_d}$ respectively.
% Let the model obtained in Theorem \ref{thm:CBrich} be denoted by $f$. Consider the projection matrix $P$ along the top singular vector of the first layer weight matrix of $f$. Then, the dataset obtained by projecting the input through $P_{\perp}$ is not separable along the linear coordinate.
\end{proposition}
% In other words, the above proposition shows that while the first model $f$ relies solely on the $1^\textrm{st}$ coordinate, the second model $f_\textrm{proj}$ obtained by $\diverse$ relies solely on $2^\textrm{nd} - d^\textrm{th}$ coordinates. Consequently, by the definition of \ifm, 

\textbf{Diversity Metrics}: Given any two models $f$ and $\ft$, we empirically evaluate their diversity using two metrics. The first is mistake diversity:
% \praneeth{Let's be consistent in using $f$ (as a function) to denote the models and not $M$.}
% \praneeth{We are using $f$ to denote logits in some places and labels in some places.}
% \[
$\mdiv{f, \ft} \defeq 1 - \frac{|\{i: f(\x^{(i)})\neq y^{(i)} \text{ }\&\text{ }  \ft(\x^{(i)})\neq y^{(i)}\}|}{\min(|\{i: f(\x^{(i)})\neq y^{(i)}\}|, |\{i: \ft(\x^{(i)})\neq y^{(i)}\}|}$,
% \]
where we abuse notation by using $f(x_i)$ (resp. $\ft(x_i)$) to denote the class predicted by $f$ (resp $\ft$) on $x_i$. Higher $\mdiv{f,\ft}$ means that there is very little overlap in the mistakes of $f$ and $\ft$.
% $f$ and $\ft$ represent the first and second models respectively and $i$ indexes over examples. Another proxy for diversity can be the correlation between the logits of the two models. However, across classes, logits will always be highly correlated if both the models make good predictions. To account for this, we evaluate class conditioned logit correlation (CCLcorr) defined as:
The second is class conditioned logit correlation i.e., correlation between outputs of $f$ and $\ft$, conditioned on the class. More concretely,
% \[
$\cclc{f, \ft} = \frac{\sum_{y \in \Y} \corr{[f(\x_i)], [\ft(\x_i)]: y_i=y}}{|\Y|}$,
% \]
where corr$([f(\x_i)], [\ft(\x_i)]: y_i=y)$ represents the empirical correlation between the logits of $f$ and $\ft$ on the data points where the true label is $y$. Table~\ref{tab:err_div_cclc} compares the diversity of two independently trained models ($f$ and $f_\textrm{ind}$) with that of two sequentially trained models ($f$ and $f_\textrm{proj}$).
The results demonstrate that $f$ and $f_{\textrm{proj}}$ are more diverse compared to $f$ and $f_\textrm{ind}$.
% The results clearly demonstrate that training the second model after projecting out the dominant directions of the first model leads to more diverse models compared to models trained independently.

% \vspace{-0.5cm}
Figure~\ref{fig:dec_bound} shows the decision boundary of $f$ and $f_\textrm{proj}$ on $2$-dimensional subspace spanned by top two singular vectors of the weight matrix. We observe that the decision boundary of the second model is more non-linear compared to that of the first model.

\textbf{Ensembling}: Figure~\ref{fig:rich_gauss_robust} shows the variation of test accuracy with the strength of gaussian noise added to the pretrained representations of the dataset. Here,  an ensemble is obtained by averaging the logits of the two models. We can see that an ensemble of $f$ and $f_\textrm{proj}$ is much more robust as compared to an ensemble of $f$ and $f_\textrm{ind}$
% \footnote{Note that the robustness of the ensemble is worse than model1 in case of Imagenet. This might be because the new features learnt are much less robust than the original features. We indeed verify this in Figure \ref{gauss:model2:imagenet}}. 

% \praneeth{TODO:
% \begin{itemize}
%     \item The decision boundary figures still take a lot of time to load. Please compress them.
%     \item Singular value decay plot for both models.
%     \item other activation functions?
%     \item deeper networks?
% \end{itemize}
% }

\section{Discussion}\label{sec:disc}
In this work, we characterized the simplicity bias 
 (SB) exhibited by one hidden layer neural networks in terms of the low-dimensional input dependence of the model. We provided a theoretical proof of presence of low-rank SB on a general class of linearly separable datasets. We further validated our hypothesis empirically on real datasets. Based on this characterization, we also proposed OrthoP -- a simple ensembling technique -- to train diverse models and show that it leads to models with significantly better Gaussian noise robustness.

This work is an initial step towards rigorously defining simplicity bias or shortcut learning of neural networks, which is believed to be a key reason for their brittleness \citep{geirhos2020shortcut}. Providing a similar characterization for deeper networks is an important research direction, which requires {\em deeper} understanding of the training dynamics and limit points of gradient descent on the loss landscape.
\section*{Acknowledgements}

We acknowledge support from Simons Investigator Fellowship, NSF grant DMS-2134157, DARPA grant W911NF2010021, and DOE grant DE-SC0022199.

\nocite{langley00}

\bibliography{iclr2023_conference}
\bibliographystyle{icml2023}

%%%%%%%%%%%%%%%%%%%%%%%%%%%%%%%%%%%%%%%%%%%%%%%%%%%%%%%%%%%%%%%%%%%%%%%%%%%%%%%
%%%%%%%%%%%%%%%%%%%%%%%%%%%%%%%%%%%%%%%%%%%%%%%%%%%%%%%%%%%%%%%%%%%%%%%%%%%%%%%
% APPENDIX
%%%%%%%%%%%%%%%%%%%%%%%%%%%%%%%%%%%%%%%%%%%%%%%%%%%%%%%%%%%%%%%%%%%%%%%%%%%%%%%
%%%%%%%%%%%%%%%%%%%%%%%%%%%%%%%%%%%%%%%%%%%%%%%%%%%%%%%%%%%%%%%%%%%%%%%%%%%%%%%
\newpage
\appendix
\onecolumn

% \section{Appendix}
\section{Proofs for rich and lazy regime}

\subsection{Rich regime}\label{app:rich}
We restate Theorem~\ref{thm:rich} below and prove it.
\begin{theorem}\label{thm:rich-app}
For any dataset in \ifm model satisfying the conditions in Section \ref{ifm:dataset},
%stated in Section \ref{ifm:dataset}
$\gamma \geq 1$ and $f(\nu,x)$ as in \Eqref{eq:rich:f}, the distribution $\nu^* = 0.5\delta_{\theta_1} + 0.5\delta_{\theta_2}$ on $\gS^{d+1}$ is the \emph{unique} max-margin classifier satisfying~\Eqref{eq:nustar}, where $\theta_1 = (\frac{\gamma}{\sqrt{2(1+\gamma^2)}}\e_1,\frac{1}{\sqrt{2(1+\gamma^2)}},1/\sqrt{2}) ,\theta_2 = (-\frac{\gamma}{\sqrt{2(1+\gamma^2)}}\e_1,\frac{1}{\sqrt{2(1+\gamma^2)}},-1/\sqrt{2})$ and $\e_1\defeq [1,0, \cdots, 0]$ denotes first standard basis vector. In particular, this implies that if gradient flow for 1-hidden layer FCN under rich initialization in the infinite width limit with cross entropy loss converges, and satisfies the technical conditions in Theorem~\ref{thm:technical}, then it converges to $\nu^*$ satisfying $f(\nu^*, Px_1 + P_{\perp}x_2) = f(\nu^*, x_1) \forall (x_1,y_1),(x_2,y_2) \in D$,
where $P$ represents the (rank-1) projection matrix on the first coordinate.
% the condition in Theorem $\ref{thm:CBrich}$.
\end{theorem}
\begin{proof}[Proof of Theorem \ref{thm:rich-app}:]
\citep{ChizatB20} showed the following primal-dual characterization of maximum margin classifiers in \eqref{eq:nustar}:
\begin{lemma}\citep{ChizatB20}\label{lem:primdualrich}
$\nu^*$ satisfies \eqref{eq:nustar} if there exists a data distribution $p^*$ such that the following two complementary slackness conditions hold:
\begin{equation}\label{eq:nustarcs}
    \text{Supp}(\nu^*) \subseteq \argmax_{(w,b,a)\in \sS^{d+1}} \E_{(x,y)\sim p^*} y[a(\phi(\langle w,x \rangle+b))]
    \quad \text{and}
\end{equation}

\begin{equation}\label{eq:pstarcs}
    \text{Supp}(p^*) \subseteq \argmin_{(x,y)\sim \gD} y\E_{(w,b,a)\sim \nu^*} [a(\phi(\langle w,x \rangle+b))]\,. 
\end{equation}
\end{lemma}
 The plan is to construct a distribution $p^*$ that satisfies the conditions of the above Lemma. 
 
 \textbf{Uniqueness.} Note further that for a fixed $p^*$, $\E_{(x,y)\sim p^*}yf(\nu,x)$ is an upper bound for the margin $\min_{(x,y) \sim \gD} yf(\nu,x)$ of any classifier $\nu$. Hence, for uniqueness, it suffices to show that $\delta_{\theta_1},\delta_{\theta_2}$ are the unique maximizers of the objective on the RHS of \eqref{eq:nustarcs} 
 %(and hence any solution to \eqref{eq:nustar} is supported on $\delta_{\theta_1},\delta_{\theta_2}$), 
 and that the unique maximum margin convex combination of $\delta_{\theta_1},\delta_{\theta_2}$ over $\gD$ is $\nu^*$.
 
 We first describe the support $D$ of $p^*$. For $y \in \{\pm 1\}$ we generate $(x,y) \in D$ as
\[ \x_1 = \gamma y \]
\[ \forall i \in {2,..,d}, \x_i = \left\{ \begin{array}{ccc}
\pm 1 & \mbox{for}
& y=1 \\ 0 & \mbox{for} & y=-1 \\
\end{array}\right.\]
Now for $(x,y) \in D$, define \begin{equation}\label{eq:pstardef} 
p^*(x,y) = \left\{ \begin{array}{ccc}
0.5 & \mbox{for}
& y=1 \\ 0.5^d & \mbox{for} & y=-1 \\
\end{array}\right.
\end{equation}
Note that $p^*$ is supported on $2^{d-1}$ positive instances and one negative instance.
We begin by showing \eqref{eq:pstarcs}.
\begin{claim}
$p^*$ as in \eqref{eq:pstardef} satisfies \eqref{eq:pstarcs}. Further, the unique maximum margin convex combination of $\delta_{\theta_1},\delta_{\theta_2}$ is $\nu^*$.
\end{claim}
\begin{proof}
Let us find the minimizers $(x,y) \sim \gD$ of $yf(\nu,x) = y\E_{(w,b,a)\sim \nu^*} [a(\phi(\langle w,x \rangle+b))]$ for
any $\nu = \lambda \delta_{\theta_1} + (1-\lambda) \delta_{\theta_2}$, $0 \leq \lambda \leq 1$. 

%Calculating 
$yf(\nu,x)$ for $(x,y)$ with $y=-1$ (denoting $x_1$ by $-\alpha_1$, where $\alpha_1 \geq \gamma$) is
\begin{align*}
    yf(\nu,x) = -1\Bigl[&\lambda*\phi\left(\frac{\gamma}{\sqrt{2(1+\gamma^2)}}\e_1^\top (-\alpha_1\e_1) + \frac{1}{\sqrt{2(1+\gamma^2)}}\right)*\frac{1}{\sqrt{2}}\\
    &+ (1-\lambda)*\phi\left(-\frac{\gamma}{\sqrt{2(1+\gamma^2)}}\e_1^\top (-\alpha_1\e_1) + \frac{1}{\sqrt{2(1+\gamma^2)}}\right)*\frac{-1}{\sqrt{2}}\Bigr]\,,\
\end{align*}
and for $(x,y)$ with $y=1$ (denoting $x_1$ by $\alpha_2$, where $\alpha_2 \geq \gamma$) is
\begin{align*}
    yf(\nu,x) = 1\Bigl[&\lambda*\phi\left(\frac{\gamma}{\sqrt{2(1+\gamma^2)}}\e_1^\top (\alpha_2\e_1) + \frac{1}{\sqrt{2(1+\gamma^2)}}\right)*\frac{1}{\sqrt{2}} \\
    &+ (1-\lambda)*\phi\left(-\frac{\gamma}{\sqrt{2(1+\gamma^2)}}\e_1^\top (\alpha_2\e_1) + \frac{1}{\sqrt{2(1+\gamma^2)}}\right)*\frac{-1}{\sqrt{2}}\Bigr] \,.
\end{align*}
%\[1\Bigl[0.5*\phi(0.5\e_1^\top (\gamma\e_1) + 0.5)*\frac{1}{\sqrt{2}} + 0.5*\phi(-0.5\e_1^\top (\gamma\e_1) + 0.5)*\frac{-1}{\sqrt{2}}\Bigr] =\frac{\gamma+1}{4\sqrt{2}}\,,\]
As $\gamma \geq 1$, the expressions above equal $\lambda\frac{\sqrt{\gamma \alpha_1+1}}{2}$ and $(1-\lambda)\frac{\sqrt{\gamma\alpha_2+1}}{2}$ respectively, and hence are minimized at $\alpha_1 = \alpha_2 = \gamma$. Hence, the margin of $\nu$ is $\min(\lambda,1-\lambda)\frac{\sqrt{1+\gamma^2}}{2}$ which is uniquely maximized at $\lambda = 1/2$. Further for $\lambda = 1/2$, all points in $D$ have the same value of $y f(\nu,x)$.
%In the second expression above, we used that whenever $y=1$, $x_1=1$ and hence $\langle w,x \rangle = \langle w,\gamma\e_1 \rangle$ (since $w_2,w_3,\ldots,w_d$ are zero for $w \in \text{Support}(\nu^*)$).
\end{proof}
In the rest of the proof we show \eqref{eq:nustarcs}, Let us denote by $g(w,b,a) \coloneqq \E_{(x,y)\sim p^*} y[a(\phi(\langle w,x \rangle+b))]$. We show that $\delta_{\theta_1},\delta_{\theta_2}$ are the only maximizers of $g(w,b,a)$ over $\sS^{d+1}$.

 We first find $g(\theta_1),g(\theta_2)$.
 \begin{align*}g(\theta_1) &= \Pr(y=1)\cdot1\cdot\frac{1}{\sqrt{2}}\cdot\phi\left(\frac{\gamma}{\sqrt{2(1+\gamma^2)}}\e_1^T(\gamma\e_1)+\frac{1}{\sqrt{2(1+\gamma^2)}}\right) \\&+ \Pr(y=-1)\cdot-1\cdot\frac{1}{\sqrt{2}}\cdot\phi\left(\frac{\gamma}{\sqrt{2(1+\gamma^2)}}\e_1^T(-\gamma\e_1)+\frac{1}{\sqrt{2(1+\gamma^2)}}\right) =\frac{\sqrt{\gamma^2+1}}{4}\,,\end{align*}
 where the first term is because $w_2,w_3,\ldots,w_d$ are zero for $\theta_1$. Similarly, $g(\theta_2) = \frac{\sqrt{\gamma^2+1}}{4}$. We now show that $g(w,a,b) < \frac{\sqrt{\gamma^2+1}}{4}$ for $(w,a,b)\notin \{{\theta_1},{\theta_2}\}$.

We begin by showing the following simple but useful claim.

\begin{claim}
All maximizers of $g(w,b,a)$ over $\sS^{d+1}$ satisfy $|a|=1/\sqrt{2}$. 
\end{claim}
\begin{proof}
The proof essentially follows from the $1-$homogeneity of the ReLU function $\phi$ and separability of $g(w,b,a)$. Note that $g(w,b,a) = \sqrt{\|w\|^2+b^2}a\cdot g(w',b',1)$ where $\|w'\|^2+b^2=1$. Maximizing $g(w,b,a)$ is equivalent to maximizing $g(w',b',1)$ over $\sS^{d}$ and $a\sqrt{\|w\|^2+b^2}$ over $\sS^{d+1}$ respectively. The second of these has its unique maximum at $|a|=1/\sqrt{2}$, completing the proof.
\end{proof}

Now express $g(w,b,a)$ as
\begin{align} g\left(w,b,a\right) &= a\left(\Pr(y=1)\E[\phi(w^Tx+b)|y=1] - \Pr(y=-1)\E[\phi(w^Tx+b)|y=-1]\right)\notag\\
&=\frac{a}{2}\left(\E_\sigma\big[\phi(\gamma w_1+b+\sum^d_{i=2}\sigma_i w_i)\big]-\phi(b-\gamma w_1)\right)\,,\label{eq:gwba}
\end{align}
where $\sigma_i$ are independent Rademacher random variables. We have two cases on $a$:

\noindent\textbf{Case 1: $a = 1/\sqrt{2}$.}
By \eqref{eq:gwba} we have
\[g(w,b,1/\sqrt{2}) \leq \frac{1}{2\sqrt{2}}\E_\sigma\big[\phi(\gamma w_1+b+\sum^d_{i=2}\sigma_i w_i)\big]\,.\] To simplify the above, define the random variable $X = \sum^d_{i=2}\sigma_i w_i$ and denote $\gamma w_1+b$ by $\alpha$. Note that $|\alpha| = |\gamma w_1 + b|\leq \sqrt{\frac{\gamma^2+1}{2}}$ which follows from $\|w\|^2+b^2=1/2$. The expectation in the last expression above becomes \begin{align*}
\E[\phi(X+\alpha)] &= \E[(X+\alpha)\1\{X+\alpha\geq 0\}] = \E[X\1\{X\geq-\alpha\}]+\alpha\Pr(X\geq-\alpha)\\
&=\E[X\1\{X\geq\alpha\}]+\alpha(1-\Pr(X\geq \alpha))\leq \E[X\1\{X\geq\alpha\}] + \alpha \,,
\end{align*}
where the last equality follows from symmetry of $X$. Note that $\text{Var}(X) = \sum^d_{i=2} w_i^2$ which is at most $\frac{1}{2} - \frac{\alpha^2}{1+\gamma^2}$ (using $\gamma w_1+b=\alpha$ and $\|w\|^2+b^2=1/2$).
Using~\ref{lem:expec} to upper bound $\E[X\1\{X\geq\alpha\}]$ we have
\[\E[\phi(X+\alpha)]\leq  \alpha+\sqrt{\frac{1}{2} \min\left(\frac{1}{2}, \frac{\frac{1}{2} - \frac{\alpha^2}{1+\gamma^2}}{2\alpha^2}\right) \left(\frac{1}{2} - \frac{\alpha^2}{1+\gamma^2}\right)}\,.\]
We can check that the RHS of the above has its unique maximizer at $\alpha=\sqrt{\frac{{1+\gamma^2}}{2}}$ for $|\alpha| \leq \sqrt{\frac{{1+\gamma^2}}{2}}$. Hence $g(w,b,a) \leq \frac{\sqrt{1+\gamma^2}}{4}$ in this case. We are now done since any $(w_1,b)$ satisfying $\gamma w_1 + b = \sqrt{\frac{1+\gamma^2}{2}}$ and $w^2_1 + b^2 \leq 1/2$ has $b = \frac{1}{\sqrt{2(1+\gamma^2)}}$.

\noindent\textbf{Case 2: $a = -1/\sqrt{2}$.} Using \eqref{eq:gwba} we have $ g(w,b,-1/\sqrt{2}) \leq \phi(b-\gamma w_1)/2\sqrt{2}$ which for $b^2 + w_1^2 \leq 1/2$ attains its unique maximum $\sqrt{\frac{\gamma^2+1}{4}}$ at $b = \frac{1}{\sqrt{2(1+\gamma^2)}}$.

% Let us see why the above theorem implies Theorem~\ref{thm:richSB}.
Finally, note that the weights of the \emph{trained} network $(w,b,a)$ are sampled from $\nu^*$. Hence, the final claim in the theorem about $f(\nu^*, Px_1 + P_\perp x_2)$ follows since the distribution of $w$ only has a support on $\e_1$ and $-\e_1$.

\end{proof}

%Note that the margin of $f$ is $O(1)$. 
%We now show a classifier $g$ dependent only on $P^{\perp}(x)$ that achieves margin $\Omega(\sqrt{d})$ - 

\subsubsection{Auxiliary lemmas for rich regime}
\begin{lemma}\label{lem:expec}
For any symmetric discrete random variable X with bounded variance, for $\alpha > 0$,
\[ \mathbb{E}[X \mathbb{I}(X \geq \alpha)] \leq \sqrt{\frac{1}{2} \min\left(\frac{1}{2}, \frac{Var(X)}{2\alpha^2}\right) Var(X)}\,.\]
\end{lemma}
\begin{proof}
\begin{equation}\label{eq:rv}\mathbb{E}[X \mathbb{I}(X \geq \alpha)] = \sum_{x \geq \alpha} xp(x) = \sum_{x \geq \alpha} \sqrt{p(x)}\sqrt{p(x)}x \leq \sqrt{p(X \geq \alpha) \sum_{x \geq \alpha} x^2 p(x)} \,,\end{equation}
where the last inequality is by Cauchy-Schwartz. Also by Chebyshev's inequality, $p(|X| \geq \alpha) \leq Var(X)/2\alpha^2$. Combining this with \eqref{eq:rv} and using symmetry of $X$ and non-negativity of $\alpha$ gives the required lemma.
%and the last inequality is by symmetry of $X$ and non-negativity of $\alpha$. The lemma now follows from observing that $p(X \geq \alpha) \leq \min(\frac{1}{2},\frac{Var(X)}{2\alpha^2})$ by Chebyshev's inequality and symmetry of $X$.
\end{proof}

\subsubsection{Proof of proposition \ref{prop:diverse-ifm}} \label{app:prop-diverse}
We restate Proposition \ref{prop:diverse-ifm} and prove it
\begin{proposition}
Consider any \ifm dataset as described in Section~\ref{ifm:dataset}.
Let $f$ be the model described in Theorem~\ref{thm:CBrich} and $f_\textrm{proj}$ be the second model obtained by $\diverse$. Then, the outputs $f$ and $f_\textrm{proj}$ on $x$ i.e., $f(x)$ and $f_\textrm{proj}(x)$ depend only on $x_1$ and $\set{x_2, \cdots, x_d}$ respectively.
\end{proposition}

\begin{proof}
As shown in Theorem \ref{thm:CBrich}, the final distribution of the weights is given by $\nu^* = 0.5\delta_{\theta_1} + 0.5\delta_{\theta_2}$, where $\theta_1 = (\frac{\gamma}{\sqrt{2(1+\gamma^2)}}\e_1,\frac{1}{\sqrt{2(1+\gamma^2)}},1/\sqrt{2}) ,\theta_2 = (-\frac{\gamma}{\sqrt{2(1+\gamma^2)}}\e_1,\frac{1}{\sqrt{2(1+\gamma^2)}},-1/\sqrt{2})$ and $\e_1\defeq [1,0, \cdots, 0]$ denotes first standard basis vector.

As the first layer weight matrix only has support along the $\e_1$ direction, therefore its top singular vector also points along the $\e_1$ direction. Hence, $P  = \e_1 \e_1^\top$ and $P_{\perp} = I - \e_1\e_1^\top$, where $I$ denotes the identity matrix. Thus, the dataset obtained by projecting the input through $P_{\perp}$ has value $0$ for the linear coordinate, for both $y=+1$ and $y=-1$. Hence, it is not separable along the linear coordinate. Thus, the second model $f_{proj}$ relies on other coordinates for classification.
\end{proof}

\subsection{Lazy regime}\label{app:lazy}
Theorem~\ref{thm:lazyrandacc} is a corollary of the following more general theorem. \begin{theorem}\label{thm:lazy2layer}
    Consider a point $x \in D$. For sufficiently small $\epsilon > 0$, there exist an absolute constant $N$ such that for all $d > N, \gamma < \epsilon\sqrt{d}$ and $\gamma \geq 7$, for the joint training of both the layers of 1-hidden layer FCN in the NTK regime, the prediction of any point of the form $(\zeta, x_{2:d})$ satisfies the following:
        \begin{enumerate}
            \item For $\zeta \geq 0.73$, the prediction is positive.
            \item For $\zeta \leq -0.95\gamma$, the prediction is negative.
            % \item For $\zeta \geq 0.32\gamma - \frac{0.68}{\gamma}$, the prediction is positive.
        
            % \item For $\zeta \leq 0.1\gamma - \frac{0.9}{\gamma}$, the prediction is negative.
            %  \item For $\kappa\left(\frac{1-\gamma\zeta}{\rho_2\sqrt{1+\zeta^2}}\right) \leq \frac{1.11}{1.11 + \pi^2} (\approx 0.1)$, the prediction is positive.
            % \item For $\zeta \leq \frac{1}{\gamma}$, the prediction is negative.
        \end{enumerate}
\end{theorem}
The above theorem establishes that perturbing $x_1$ by $O(\gamma)$ changes $pred(f(x))$ for $x \in D$ (whereas a classifier exists that achieves a margin of $\Omega(\sqrt{d})$ on $D$, as $D$ has margin $1$ for coordinates $\{2\cdots d\}$).
  %However,  %we can change the label of any $(x,y) \in D$ by a perturbation of $O(\gamma)$ in the 1st coordinate. 
As $\gamma = o(d)$, this shows that the learned model is adversarially vulnerable.

\begin{proof}[Proof of Theorem \ref{thm:lazy2layer}]
The idea of the proof is to obtain an explicit expression for $f(x)$ by applying standard kernel max-margin SVM theory to the NTK kernel \ref{thm:lazyBM}.

We begin with some preliminaries. We will refer to the first coordinate of the instance as the 'linear' coordinate, and to the rest as 'non-linear' coordinates. 
Also, henceforth we append an extra coordinate with value $1$ to all our instances (corresponding to bias term) - as is standard for working with unbiased SVM without loss of generality.

\noindent\textbf{Explicit expression for $f$.} Using representer theorem for max margin kernel SVM, we know that $f$ can be expressed as
\[ f(x) = \sum_{(x^{(t)},y^{(t)})\in D} \lambda_t y^{(t)}K(x,x^{(t)})\,, \]
for some $\lambda_t \geq 0$ (that are known as \emph{Lagrange multipliers}). Further by KKT conditions, a function possessing such a representation (that correctly classifies $D$) has maximum margin if $y^{(t)}f(x^{(t)}) = 1$ whenever $\lambda_t > 0$ (training points $t$ satisfying $\lambda_t > 0$ are called \emph{support vectors}).

We begin with a useful claim.
\begin{claim}\label{cl:supp}
The max margin kernel SVM for $D$ with the NTK kernel has all points in $D$ as support vectors.
\end{claim}
\begin{proof}
By the above discussion, it suffices to show that the (unique) solution $\alpha \in \R^{|D|}$ to $K\alpha = y$ satisfies $\text{sign}(\alpha_i) = y^{(i)}$ for all $i$, where $K$ is the $|D| \times |D|$ Gram matrix with $(i,j)$th entry $K(x^{(i)},x^{(j)})$ and $y_i = y^{(i)}$ (the Lagrange multipliers $\lambda_i$ are then given by $y_i \alpha_i$).

\emph{Structure of Gram matrix.} Order $D$ so that the positive instances appear first. Then the Gram matrix $K$ has a block structure of the form
$\begin{pmatrix}
B & C \\
C^T & R
\end{pmatrix}$
where $B \in \R^{2^{d-1} \times 2^{d-1}}$ and $R \in \R$ are the Gram matrices for the positive and negative instances respectively, and $C \in \R^{2^{d-1} \times 1}$ represents the $K(x^{(i)},x^{(|D|)})$ values for $i < |D|$.

Recall that for the NTK kernel, $K(x^{(i)},x^{(j)})$ has the form $\|x^{(i)}\|\|x^{(j)}\|\kappa(\langle x^{(i)},x^{(j)} \rangle)$. Note all the positive instances have the same norm (denoted by $\rho_1 = \sqrt{d+\gamma^2}$) and %the $\kappa$ of
the inner product between two positive instances depends only on the number $i$ of non-matching non-linear coordinates (denoted by $\beta_i$ for $0 \leq i \leq d-1$). Hence, the rows of $B$ are permutations of each other, with the entry $\rho_1^2\beta_i$ appearing $d-1 \choose i$ times. Similarly, the entries in $C$ are all equal and are denoted by $\rho_1 \rho_2 \beta_d$ where $\beta_d$ denotes $\kappa(x^{(t)},x^{|D|})$ for any $t < |D|$ and $\rho_2 = \|x^{|D|}\| = \sqrt{1+\gamma^2}$. The only entry in $R$ is $\rho_2^2 \kappa(1)$. In particular,
\[\beta_i = \kappa\left(\frac{d - 2i + \gamma^2}{d + \gamma^2}\right) \text{ for } i \in [|D|-1],\qquad \text{and}\qquad \beta_d = \kappa\left(\frac{1 - \gamma^2}{\sqrt{d + \gamma^2}\sqrt{1 + \gamma^2}}\right)\,.\]

Now we are ready to solve $K\alpha = y$. By symmetry in the structure of K, $\alpha$ looks like $[a, a, ......, b]$, where the first $|D|-1$ entries are the same.

Expanding $K\alpha = y$, we get two equations given by
\[ a\rho_1^2\left(\sum_{i=0}^{d-1} {d-1 \choose i} \beta_i\right) + b\rho_1\rho_2\beta_d = 1 \qquad\text{and}\qquad2^{d-1}a\rho_1\rho_2\beta_d + \rho_2^2\kappa(1)b = -1\,. \]
Solving, we get
\[ a = \frac{\rho_2\kappa(1) + \rho_1\beta_d}{\rho_1^2 \rho_2\sum_{i=0}^{d-1} \left({d-1 \choose i} [\kappa(1)\beta_i - \beta_d^2]\right)} \qquad\text{and}\qquad b = \frac{-1 - 2^{d-1}a\rho_1 \rho_2 \beta_d}{\rho_2^2\kappa(1)}\,. \]
We now show that $a > 0$ and $b < 0$. Note that for sufficiently large $d$, $\beta_d$ can be made arbitrarily close to $\kappa(0) = 1/\pi$ (since $\kappa$ is smooth around $0$). Hence, $a > 0$ implies $b < 0$. We in fact give the following estimate for $a$:
\begin{equation}\label{eq:a}
a = 2^{1-d}\cdot\frac{\rho_2\kappa(1)+\rho_1\beta_d}{\xi\rho^2_1\rho_2} \qquad\text{where}\qquad \frac{2}{\pi}-\frac{1}{\pi^2} + O\left(\frac{1}{d}\right)\leq \xi \leq 2 + O\left(\frac{1}{d}\right)\,.
\end{equation}
For the lower bound on $\xi$, write
\begin{align*}
    &\sum_{i=0}^{d-1} {d-1 \choose i} [\kappa(1)\beta_i - \beta_d^2] 
    = \kappa(1)\sum^{\lfloor d/2 \rfloor}_{i=0} {d-1 \choose i} (\beta_i + \beta_{d-1}) -  2^{d-1}\beta^2_d\\
    &\geq\kappa(1)\sum^{\lfloor d/2 \rfloor}_{i=0} {d-1 \choose i} 2 \beta_{d/2} -  2^{d-1}\beta^2_d \geq 2^{d-1}\left(\kappa(1)\kappa(0) - \kappa^2(0) + O\left(\frac{1}{d}\right)\right)\,,
\end{align*}
where for the first inequality we used convexity of $\kappa$ and for the second inequality we used $\beta_{d/2} = \kappa(0) + O(1/d), \beta_d = \kappa(0) + O(1/\sqrt{d})$.  
For the upper bound on $\xi$, write
\begin{align*}
    &\sum_{i=0}^{d-1} {d-1 \choose i} [\kappa(1)\beta_i - \beta_d^2] 
    \leq\kappa(1)\sum^{d-1}_{i=0}{d-1 \choose i}\kappa \left(1-\frac{2i}{d+\gamma^2}\right)\\
    &\leq\kappa(1)\sum^{d-1}_{i=0}{d-1 \choose i} \left(2-\frac{2i}{d+\gamma^2}\right)
    =\kappa(1)2^{d}-\frac{\kappa(1)(d-1)2^{d-1}}{d+\gamma^2}\,,
\end{align*}
where for the second inequality we used $\kappa(u) \leq 1+u$ (which holds by convexity and $\kappa(-1)=0, \kappa(1)=2$).
\end{proof}

Now we analyze predicted labels for points of the form $(\zeta,x_{2:d+1})$ where $x \in D$. We make two cases depending on the label of $x$.

\textbf{Predicted label for point $(\zeta,x^{(t)}_{2:d+1})$ where $x^{(t)} \in D$ has positive label}

Our point (denoted by $x$) has the form $(\zeta,\zeta_1,\zeta_2,\ldots,\zeta_d,1)$ where $\zeta_i \in \pm 1$. The idea of the proof is to write $f$ explicitly as a function of $\zeta$ and work with its first order Taylor expansion around $\zeta = \gamma$, with some additional work to take care of non-smoothness of $f$.

\emph{Explicit form for $f$.} Let $\tau_i \defeq \langle x,x' \rangle/(\|x\|\|x'\|)$ for a positive instance $x' \in D$, where $x$ and $x'$ have exactly $i$ non-matching non-linear coordinates (for $0 \leq i \leq d-1$). Similarly denote by $\tau_d$ the quantity $\langle x,x^{|D|} \rangle/(\|x\|\|x^{|D|}\|)$. In particular,
\[ \tau_i = \left(\frac{d-2i+\gamma\zeta}{\rho_1\|x\|}\right)\qquad \text{and} \qquad \tau_d = \left(\frac{1-\gamma\zeta}{\rho_2\|x\|}\right)\,. \]

By the above discussion, we have
\begin{align*}
f(x) = a\left(\sum^{|D|-1}_{t=1}K(x,x^{(t)})\right)+bK(x,x^{|D|})
= a\rho_1\|x\|\left(\sum^{d-1}_{i=0} {d-1 \choose i}\kappa(\tau_i)\right)+b\rho_2\|x\|\kappa(\tau_d)\,.
\end{align*}
Substituting $b$ and denoting $f(x)/\|x\|$ by $g(\zeta)$ we get
\begin{equation}\label{eq:gzeta}g(\zeta) = a\rho_1 \left[ \sum^{d-1}_{i=0} {d-1 \choose i} \kappa(\tau_i(\zeta)) - \frac{2^{d-1}\beta_d}{\kappa(1)} \kappa(\tau_d(\zeta)) \right] - \frac{\kappa(\tau_d(\zeta))}{\rho_2\kappa(1)}\,.\end{equation}

Now try to expand $g(\zeta)$ using the Taylor series around $\zeta = \gamma$ (note that $g(\gamma) = 1/\rho_1$). Note that $\kappa'$ can however be unbounded around $-1$ and $1$. To get around this, write $g = h + q$, where $h$ has bounded first and second derivative, and $q$ has lower order than $h$ for $\zeta$ of interest. In particular,
\[ h(\zeta) = a\rho_1 \left[ \sum^{3d/4}_{i=d/4} {d-1 \choose i} \kappa(\tau_i(\zeta)) - \frac{2^{d-1}\beta_d}{\kappa(1)} \kappa(\tau_d(\zeta)) \right] - \frac{\kappa(\tau_d(\zeta))}{\rho_2\kappa(1)}\qquad \text{and} \]
\[ q(\zeta) = a\rho_1 \left[ \sum_{i:|d/2-i|>d/4} {d-1 \choose i} \kappa(\tau_i(\zeta)) \right]\,.\]
Observe that $q(\zeta) = o(c^{d})$ for $c < 1$ using the estimate \eqref{eq:a} for $a$ and concentration for sums of independent Bernoullis. By Taylor's theorem,
\begin{equation}\label{eq:taylor}
 g(\zeta) = h(\gamma) + h'(\gamma)(\zeta-\gamma) + \frac{h''(\theta)(\zeta-\gamma)^2}{2} + q(\zeta)\,,
 \end{equation}
for some $\theta \in [\gamma,\zeta]$, where $h(\gamma) \approx 1/\sqrt{d}$. It will turn out that $|h'(\gamma)| = \Theta(1/\sqrt{d})$, $|h''(\zeta)| = o(1/\sqrt{d})$. This will allow us to complete the proof using the linear approximation of $g(\zeta)$ by neglecting the second order term and $q(\zeta)$. We now compute $h',h''$, treating $\|x\| = \sqrt{d+\zeta^2}$ as a constant for exposition (the proof works without this approximation or the reader may think of $\gamma$ as $o(\sqrt{d})$).
%$g(\zeta)$ is nearly linear in $\zeta$
%It will suffice to restrict ourselves to the regime $|h(\zeta)| = \Omega(1/\sqrt{d})$ where $q(\zeta)$ can be neglected.
Using $\tau'_i(\zeta)\approx \frac{\gamma}{\rho_1\|x\|},\tau'_d(\zeta)\approx \frac{-\gamma}{\rho_2\|x\|}$,
\begin{align*}
h'(\zeta) &\approx a\rho_1\left[\sum^{d-1}_{i=0}{d-1 \choose i}\kappa'(\tau_i(\zeta))\frac{\gamma}{\rho_1\|x\|} +\frac{2^{d-1}\beta_d}{\kappa(1)}\kappa'(\tau_d(\zeta))\frac{\gamma}{\rho_2\|x\|}\right]+\frac{\kappa'(\tau_d(\zeta))}{\rho_2\kappa(1)}\frac{\gamma}{\rho_2\|x\|}\\
h''(\zeta) &\approx a\rho_1\left[\sum^{d-1}_{i=0}{d-1 \choose i}\kappa''(\tau_i(\zeta))\frac{\gamma^2}{\rho^2_1\|x\|^2} -\frac{2^{d-1}\beta_d}{\kappa(1)}\kappa''(\tau_d(\zeta))\frac{\gamma^2}{\rho_2^2\|x\|^2}\right]-\frac{\kappa''(\tau_d(\zeta))}{\rho_2\kappa(1)}\frac{\gamma^2}{\rho_2^2\|x\|^2}\,.
%&\approx\frac{\xi\beta_d\kappa'(\tau_d(\zeta))\gamma}{\kappa(1)\rho_2^2\|x\|} +\frac{\kappa'(\tau_d(\zeta))\gamma}{\rho_2^2\kappa(1)\|x\|}
%&\approx \frac{\gamma\kappa'(\tau_d(\zeta))}{\rho_2^2\kappa(1)\|x\|}(1+\beta_d\xi)\,,
\end{align*}
Plugging $\|x\|\approx \rho_1 \approx \sqrt{d}$ and substituting $a$ from \eqref{eq:a},
\[ h'(\zeta) = \frac{(1+\beta^2_d/\xi)\kappa'(\tau_d(\zeta))\gamma}{\rho^2_2\kappa(1)\sqrt{d}}+ o\left(\frac{1}{\sqrt{d}}\right)\qquad\text{and}\qquad h''(\zeta) = O\left(\frac{1}{d}\right)\,,\]
which substituted in \eqref{eq:taylor} with $\tau_d(\zeta) \approx 0, \beta_d \approx \kappa(0), \kappa'(\tau_d(\zeta)) \approx \kappa'(0)$ gives
\[ g(\zeta) = \frac{1}{\sqrt{d}}\left(1+\frac{(1+\kappa^2(0)/\xi)\kappa'(0)\gamma}{\kappa(1)\rho^2_2}(\zeta-\gamma)\right) + o\left(\frac{1}{\sqrt{d}}\right)\,,\]

%\[ g(\zeta) = \frac{1}{\sqrt{d}} + \frac{(1+\kappa^2(0)\xi)\kappa'(0)\gamma}{\kappa(1)\rho^2_2\sqrt{d}}(\zeta-\gamma) + o\left(\frac{1}{d}\right)\,,\]
%where we neglected $o(1/\sqrt{d})$ terms in $h'(\zeta),g(\zeta)$. 
Hence, $g(\zeta) > 0$  whenever the coefficient of $1/\sqrt{d}$ above is bounded above zero, and a similar condition holds for $g(\zeta) < 0$.
%\[ 1 + \frac{(1+\kappa^2(0)\xi)\kappa'(0)\gamma}{\kappa(1)\rho^2_2}(\zeta-\gamma) > \delta\]
%where $\delta$ is arbitrarily small for large $d$. A similar condition holds for $g(\zeta) < 0$. 
Using the estimates of $\xi$ from \eqref{eq:a} and $\kappa'(0) = 1,\kappa(0) = 1/\pi, \kappa(1) = 2, \rho^2_2 = 1 + \gamma^2$ in the above gives that $g(\zeta) > 0$ for $\zeta > -0.68 \gamma - 1.68/\gamma$ and $g(\zeta) < 0$ for $\zeta < -0.905\gamma-1.905/\gamma$.
%Full derivative if needed:
%\begin{align*}
%h'(\zeta) = &a\rho_1\left[\sum^{d-1}_{i=0}{d-1 \choose i}\kappa'(\tau_i(\zeta))\left(\frac{\gamma}{\rho_1\|x\|}-\frac{\zeta(d-2i+\gamma\zeta)}{\|x\|^{3}\rho_1}\right) - \frac{2^{d-1}\beta_d}{\kappa(1)}\kappa'(\tau_d(\zeta))\left(-\frac{\gamma}{\rho_2\|x\|}-\frac{\zeta(1-\gamma\zeta)}{\|x\|^3\rho_2}\right)\right]\\
%&- \frac{\kappa'(\tau_d(\zeta))}{\rho_2\kappa(1)}\left(-\frac{\gamma}{\rho_2\|x\|}-\frac{\zeta(1-\gamma\zeta)}{\|x\|^3\rho_2}\right)\approx \frac{a\rho_1 2^{d-1}\kappa'(\tau_d(\zeta))\gamma}{\kappa(1)\rho_2\|x\|}
%+ \frac{\kappa'(\tau_d(\zeta))\gamma}{\rho_2^2\kappa(1)\|x\|} \approx \frac{\xi\beta_d\kappa'(\tau_d(\zeta))\gamma}{\kappa(1)\rho_2^2\|x\|} + \frac{\kappa'(\tau_d(\zeta))\gamma}{\rho_2^2\kappa(1)\|x\|} \,,
%\end{align*} 
%Further,
%\[h''(\zeta) \approx \kappa''(\tau_d(\zeta)) \frac{\xi\beta_d\kappa'(\tau_d(\zeta))\gamma}{\kappa(1)\rho_2^2\|x\|} +\frac{\kappa'(\tau_d(\zeta))\gamma}{\rho_2^2\kappa(1)\|x\|} \]

\textbf{Predicted label for point $(\zeta,x^{(t)}_{2:d+1})$ where $x^{(t)} \in D$ has negative label}

Following the same plan, write our point (denoted by $x$) as $(\zeta,0,\ldots,0,1)$.

\emph{Explicit form for $f$.}
Begin by finding
\[ \tau_i = \left(\frac{1+\gamma\zeta}{\rho_1\|x\|}\right)\qquad \text{and} \qquad \tau_d = \left(\frac{1-\gamma\zeta}{\rho_2\|x\|}\right)\,. \]

\eqref{eq:gzeta} now gives
\[g(\zeta)=2^{d-1}a\rho_1\left[\kappa(\tau_0(\zeta)) - \frac{\beta_d\kappa(\tau_d(\zeta))}{\kappa(1)}\right]-\frac{\kappa(\tau_d(\zeta))}{\rho_2\kappa(1)}\,.\]

Expanding $\kappa(\tau_0(\zeta))$ using Taylor series around $\zeta = -1/\gamma$,
\[\kappa(\tau_0(\zeta)) = \kappa(0) + \kappa'(\tau_0(\theta))\tau'_0(\theta)(\zeta+\frac{1}{\gamma})\,, \]
for some $\theta \in [-1,1]$. For large $d$, $\tau_0(\theta) \approx 0$ and $\tau'_0(\theta) = O(1/\sqrt{d})$. Hence we have
\begin{align*}
g(\zeta)&=\frac{\rho_2\kappa(1)+\rho_1\beta_d}{\xi \rho_1 \rho_2}\left[\kappa(0)+O\left(\frac{1}{\sqrt{d}}\right) -\frac{\beta_d\kappa(\tau_d(\zeta))}{\kappa(1)}\right]-\frac{\kappa(\tau_d(\zeta))}{\rho_2\kappa(1)}\\
&=\frac{1}{\rho_2}\left(\frac{\kappa^2(0)}{\xi} - \left(\frac{\kappa^2(0)}{\xi \kappa(1)}+\frac{1}{\kappa(1)}\right)\kappa(\tau_d(\zeta))\right) + o(1)\,.
\end{align*}
As before $g(\zeta) > 0$ whenever the coefficient of $1/\rho_2$ above is bounded above zero which happens for $\zeta \geq 0.73$ (for $\gamma \geq 3$). Similarly, $g(\zeta) < 0$ for $\zeta \leq 0$.\end{proof}
\section{Experiments}\label{app:exp}
In this section, we provide experimental details, including hyperparameter tuning setup and some additional experiments.
%\praneeth{Depen: TODO}
\subsection{Details on the experimental setting} \label{app:exp-setting}
We will first describe the four datasets that have been used in this work.

\begin{enumerate}
    \item \textbf{Imagenette} \citep{imagenette}: This is a subset of 10 classes of Imagenet, that are comparatively easier to classify.
    \item \textbf{b-Imagenette}: This is a binarized version of Imagenette, where only a subset of two classes (tench and English springer) is used. 
    \item \textbf{Waterbirds-Landbirds} \citep{Sagawa*2020Distributionally}: This is a majority-minority group dataset, consisting of waterbirds on water and land background, as well as landbirds on land and water background.
    This dataset serves as a baseline for checking the dependence of model on the spurious background feature when predicting the bird class, as most of the training examples have waterbirds on water and landbirds on land background.
    %\item \textbf{Imagenet} \citep{deng2009imagenet}: This is the standard benchmark for large scale image classification.
    \item \textbf{MNIST-CIFAR} \citep{shah2020pitfalls}: This is a collage dataset, created by concatenating MNIST and CIFAR images along an axis. This is a synthetic dataset for evaluating the simplicity bias of a trained model.
\end{enumerate}

\paragraph{Setup} 

Throughout the paper, we work with the pretrained representations of the above datasets, obtained by using an Imagenet pretrained Resnet 50. We finetune a 1-hidden layer FCN with a hidden dimension of $100$ on top of these representations (keeping the backbone fixed) using SGD with a momentum of 0.9. Every model is trained for $20000$ steps with a warmup and cosine decay learning rate scheduler. For each of the runs, we tune the batch size, learning rate and weight decay using validation accuracy. Below are the hyperparameter tuning details:
\begin{itemize}
    \item Batch size $\in \{128, 256\}$
    \item Learning rate:  
    \begin{itemize}
        \item Rich regime: $\in \{0.5, 1.0\}$ (as learning rate in rich regime needs to scale up with the hidden dimension)
        \item Lazy regime: $\in \{0.01, 0.05\}$
    \end{itemize} 
    \item Weight decay: $\in \{0, 1e^{-4}\}$
\end{itemize}

The final numbers reported are averaged across 3 independent runs with the selected hyperparameters.

\paragraph{Evaluation}

For Imagenette, b-Imagenette and MNIST-CIFAR, we report the standard test accuracy in all the experiments. For waterbirds, we report train-adjusted test accuracy, as reported in \citet{Sagawa*2020Distributionally}. Precisely, accuracy for each group present in the test data is individually calculated and then weighed by the proportion of the corresponding group in the train dataset.

\subsection{Additional experimental results}
\label{app:additional-exps}

In this section, we present a few additional experimental results. 

\paragraph{Accuracy of $f_{\textrm{proj}}$} In Table \ref{tab:f_proj_acc} and \ref{tab:f_proj_acc_lazy}, we show the test accuracy of $f_{\textrm{proj}}$ in rich and lazy regime respectively. As can be seen, even after projecting out the principal components used by $f$, $f_{\textrm{proj}}$ attains significantly high accuracy. Note that, in these experiments, model 1 was kept fixed and the accuracy of $f_{\textrm{proj}}$ is averaged across 3 runs.

% \paragraph{Results on Imagenet} The evolution of effective rank of the first layer weight matrix is shown in Figure \ref{fig:eff_rank_imagenet}. As can be seen, the weight matrix becomes sufficiently low rank as the training progresses. 

\begin{table}[t]
\caption{Trained accuracy of $f_{\textrm{proj}}$ in rich regime
}
\label{tab:f_proj_acc}
\begin{center}
\setlength\tabcolsep{4.5pt}
\begin{tabular}{|M{2.2cm} M{2cm} M{2cm}|} 
 \hline
 Dataset & Acc($f$) & Acc($f_{\textrm{proj}}$) \\ [0.5ex] 
 \hline\hline
%  Dataset & $\rank{P}$ & Acc($x$) & Acc($P$) & Acc($P_{\perp}$) & $100\frac{\| f(x) - f(Px)\|}{\|f(x)\|}$ & $P$ LC:\\ [0.5ex] 
%  \hline\hline
 b-Imagenette& $93.35$ & $91.35\pm0.32$ \\ 
 Imagenette& $79.67$ & $71.93\pm0.12$ \\ 
 Waterbirds& $90.29$ & $89.92\pm0.08$ \\
 \mnistcifar& $99.69$ & $98.95\pm0.02$ \\ 
 %Imagenet & $72.02$ & $69.63\pm 0.08$ \\ [1ex] 
 \hline
\end{tabular}
\end{center}
\end{table}

\begin{table}[t]
\caption{Trained accuracy of $f_{\textrm{proj}}$ in lazy regime
}
\label{tab:f_proj_acc_lazy}
\begin{center}
\setlength\tabcolsep{4.5pt}
\begin{tabular}{|M{2.2cm} M{2cm} M{2cm}|} 
 \hline
 Dataset & Acc($f$) & Acc($f_{\textrm{proj}}$) \\ [0.5ex] 
 \hline\hline
%  Dataset & $\rank{P}$ & Acc($x$) & Acc($P$) & Acc($P_{\perp}$) & $100\frac{\| f(x) - f(Px)\|}{\|f(x)\|}$ & $P$ LC:\\ [0.5ex] 
%  \hline\hline
 b-Imagenette& $93.09$ & $91.77\pm0.34$ \\ 
 Imagenette& $80.31$ & $77.34\pm0.21$ \\ 
 Waterbirds& $90.4$ & $89.5\pm0.18$ \\
 \mnistcifar& $99.74$ & $98.54\pm0.00$ \\ 
 %Imagenet & $72.6$ & $72.07\pm 0.08$ \\ [1ex] 
 \hline
\end{tabular}
\end{center}
\end{table}

\paragraph{Singular value decay}. In Figure \ref{fig:sing_value}, we provide the singular value decay of the weight matrix for the first model trained in rich regime. As can be seen, the top few singular values capture most of the Frobenius norm of the matrix.

\begin{figure}[t]
\centering
\includegraphics[width=0.5\textwidth]{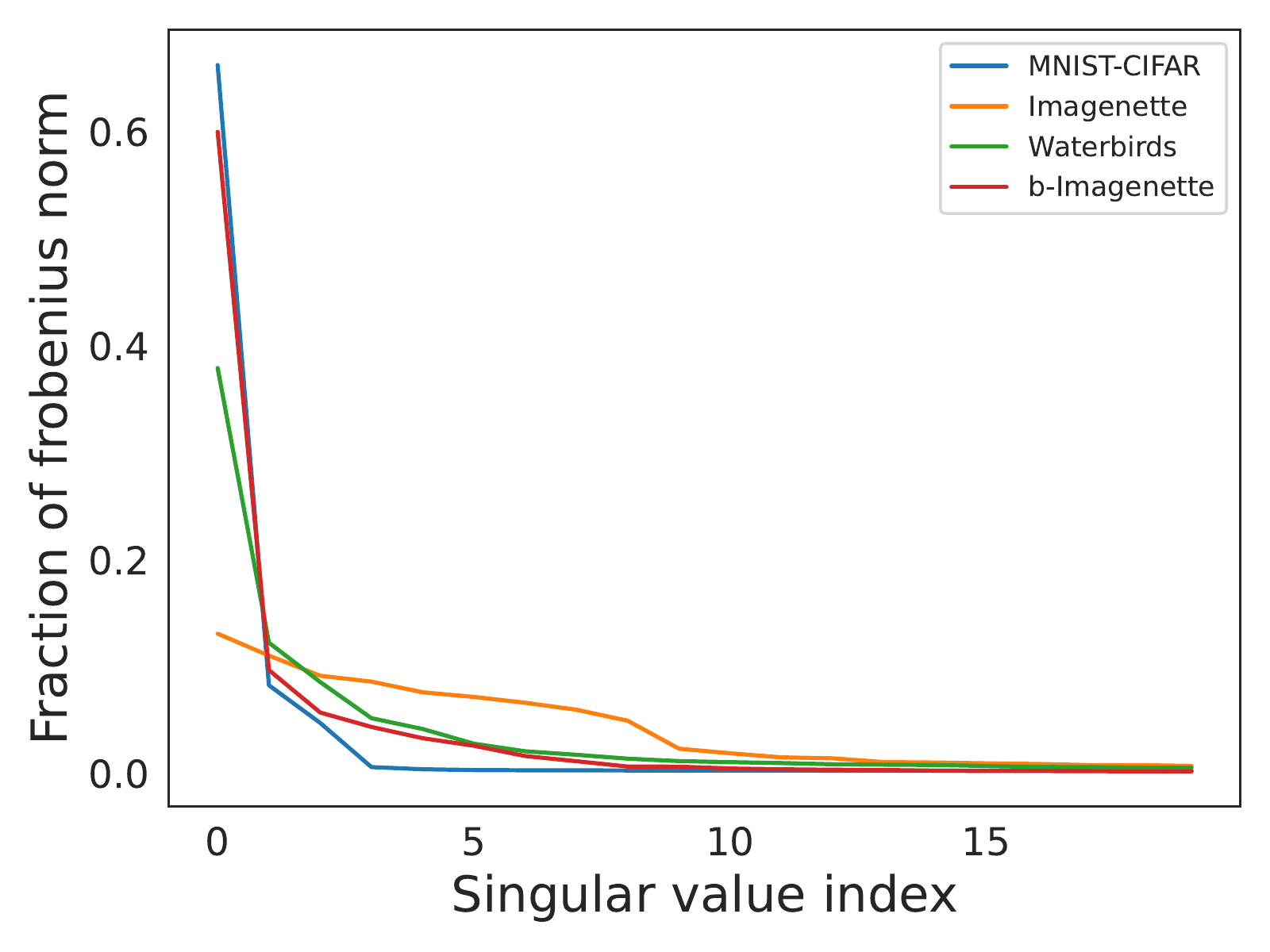}
\caption{Fraction of Frobenius norm captured by the top $i^{th}$ singular value i.e., $\sigma_i^2 / \sum_{j=1}^d \sigma_j^2$ vs $i$ of the first layer weight matrix trained in rich regime for various datasets.}
\label{fig:sing_value}
\end{figure}

\paragraph{MNIST-CIFAR} In Figure \ref{fig:rich_mnist_gauss_robust}, we show that an ensemble of $f$ and $f_{\textrm{proj}}$ has better gaussian robustness than an ensemble of $f$ and $f_{\textrm{ind}}$ on MNIST-CIFAR dataset.

\begin{figure}[t]
\centering
\includegraphics[width=0.5\textwidth]{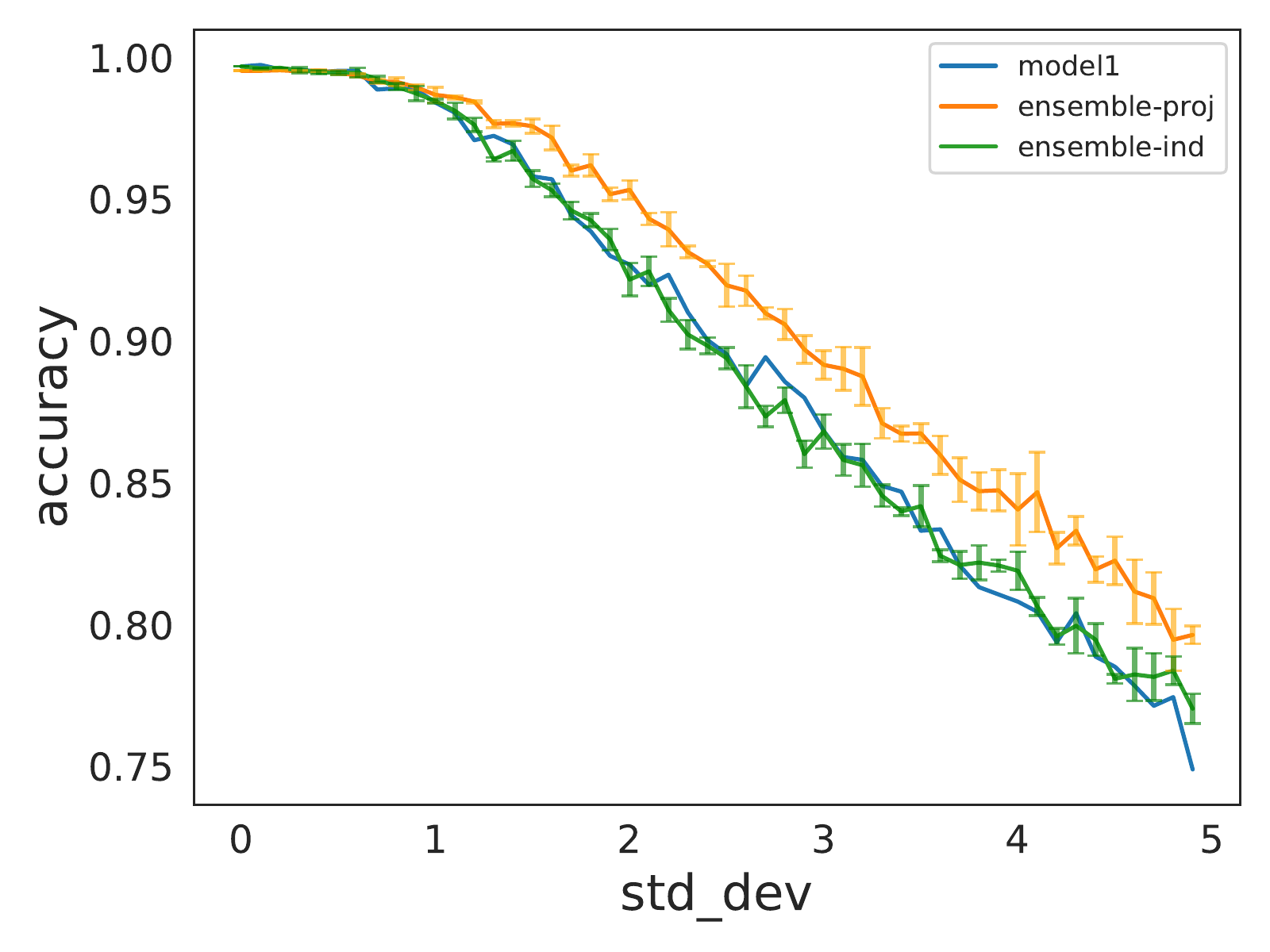}
% \begin{subfigure}{.5\textwidth}
%   \centering
%   \includegraphics[width=\textwidth]{plots/mnist-cifar/gauss_robust_test_overall_3_plots.pdf}
%     \caption{MNIST-CIFAR}
%     \label{mnist-cifar:eff_rank}
% \end{subfigure}
\caption{Variation of test accuracy with the standard deviation of Gaussian noise added to the pretrained representations of MNIST-CIFAR dataset. Model 1 is kept fixed, and values for both the ensembles are averaged across 3 runs.}
\label{fig:rich_mnist_gauss_robust}
\end{figure}

\paragraph{Quantitative measurement of non-linearity of decision boundary} In this section, we report a quantitative measure of non-linearity of the decision boundary along the top two singular vectors for $f$ and $f_{\textrm{proj}}$. Basically, we fit a linear classifier to the decision boundary and report its accuracy. As shown in Table \ref{tab:quant_non_lin}, the test accuracy obtained by the linear classifier for $f_{\textrm{proj}}$ is less than $f$.

\begin{table}[t]
\caption{Quantitative measurement of non-linearity of decision boundary -- accuracy of fitted linear classifier to the decision boundary
}
\label{tab:quant_non_lin}
\begin{center}
\setlength\tabcolsep{4.5pt}
\begin{tabular}{|M{2.2cm} M{4cm} M{4cm}|} 
 \hline
 Dataset & Linear-Classifier-Acc($f$) & Linear-Classifier-Acc($f_{\textrm{proj}}$) \\ [0.5ex] 
 \hline\hline
%  Dataset & $\rank{P}$ & Acc($x$) & Acc($P$) & Acc($P_{\perp}$) & $100\frac{\| f(x) - f(Px)\|}{\|f(x)\|}$ & $P$ LC:\\ [0.5ex] 
%  \hline\hline
 b-Imagenette& $96.12$ & $95.28\pm 0.2$ \\ 
 Waterbirds& $97.28$ & $93.24\pm0.24$ \\
 \hline
\end{tabular}
\end{center}
\end{table}

\paragraph{Variation of LD-SB with depth} In Figure \ref{fig:depth-2} and \ref{fig:depth-3}, we show the evolution of effective rank of weight matrices for depth-2 and 3 ReLU networks. As can be seen, the rank still decreases with training, however the effect is less pronounced for the initial layers. Note that the initialization used in these runs was the feature learning initialization as proposed in \citet{Yang21}.

\begin{figure}[t]
\centering
\begin{subfigure}{.4\textwidth}
  \centering
  \includegraphics[width=\textwidth]{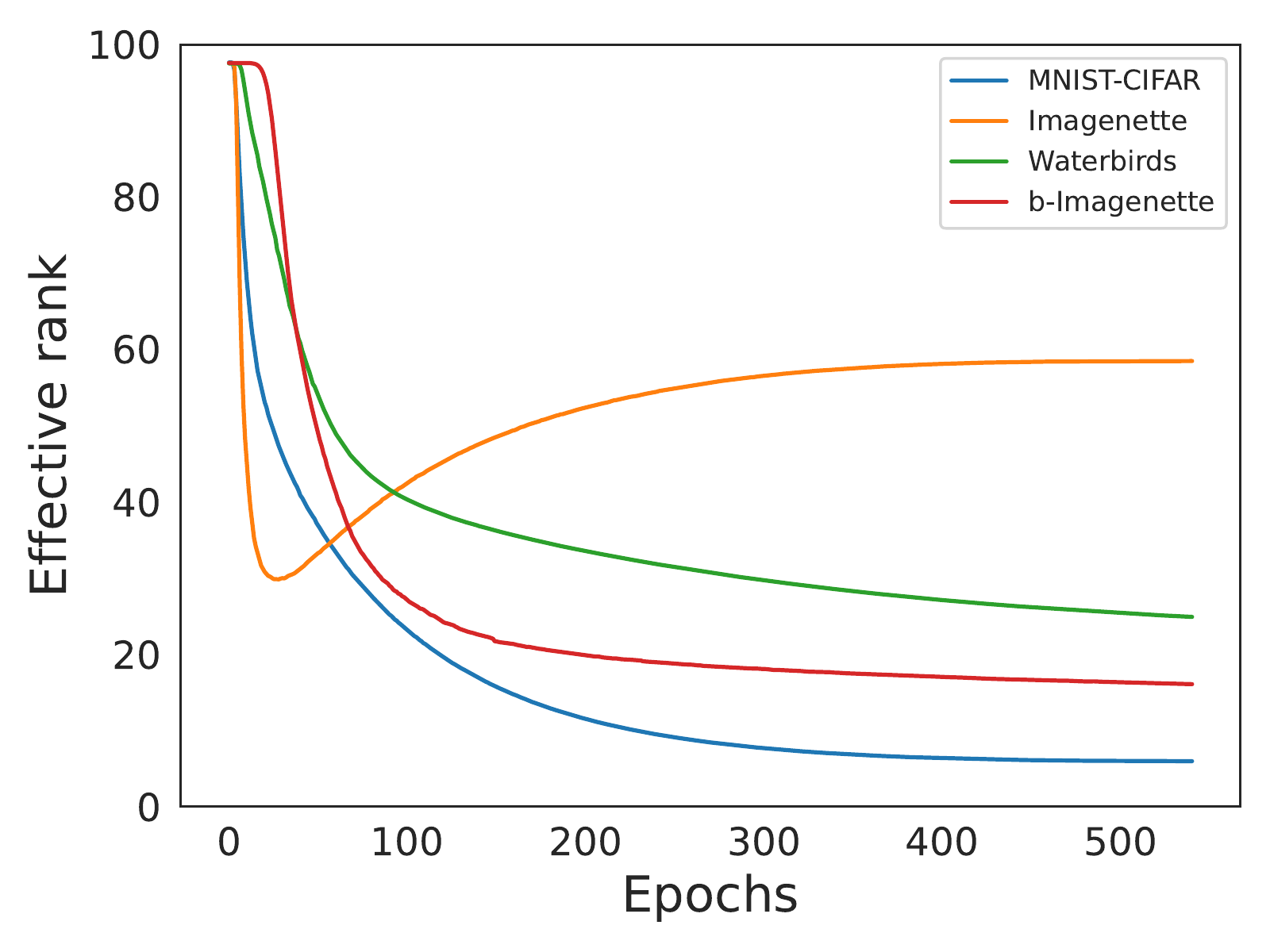}
    \caption{Layer 1}
    %\label{b-imagenet:eff_rank}
\end{subfigure}
\begin{subfigure}{.4\textwidth}
  \centering
  \includegraphics[width=\textwidth]{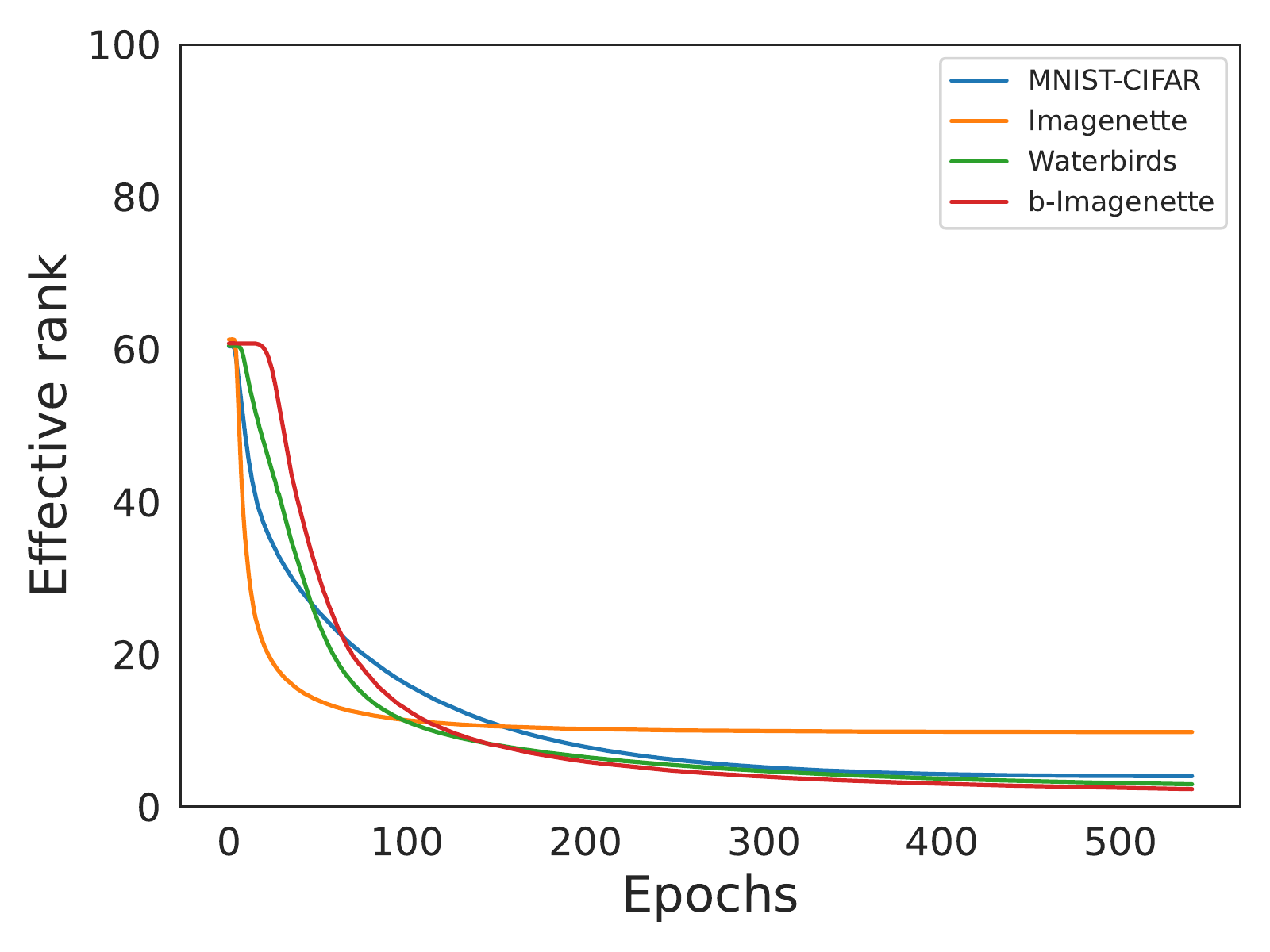}
    \caption{Layer 2}
    %\label{imagenet:eff_rank}
\end{subfigure}
%\begin{subfigure}{.3\textwidth}
%  \centering
%  \includegraphics[width=\textwidth]{plots/waterbirds/gauss_robust_test_overall_3_plots_stddev.pdf}
%    \caption{Waterbirds-Landbirds}
%    \label{waterbirds:eff_rank}
%\end{subfigure}         
% \begin{subfigure}{.5\textwidth}
%   \centering
%   \includegraphics[width=\textwidth]{plots/mnist-cifar/gauss_robust_test_overall_3_plots.pdf}
%     \caption{MNIST-CIFAR}
%     \label{mnist-cifar:eff_rank}
% \end{subfigure}
\caption{Evolution of effective rank of the weight matrices for a depth-2 ReLU network on Resnet-50 pretrained representations of the dataset}
\label{fig:depth-2}
\end{figure}

\begin{figure}[t]
\centering
\begin{subfigure}{.3\textwidth}
  \centering
  \includegraphics[width=\textwidth]{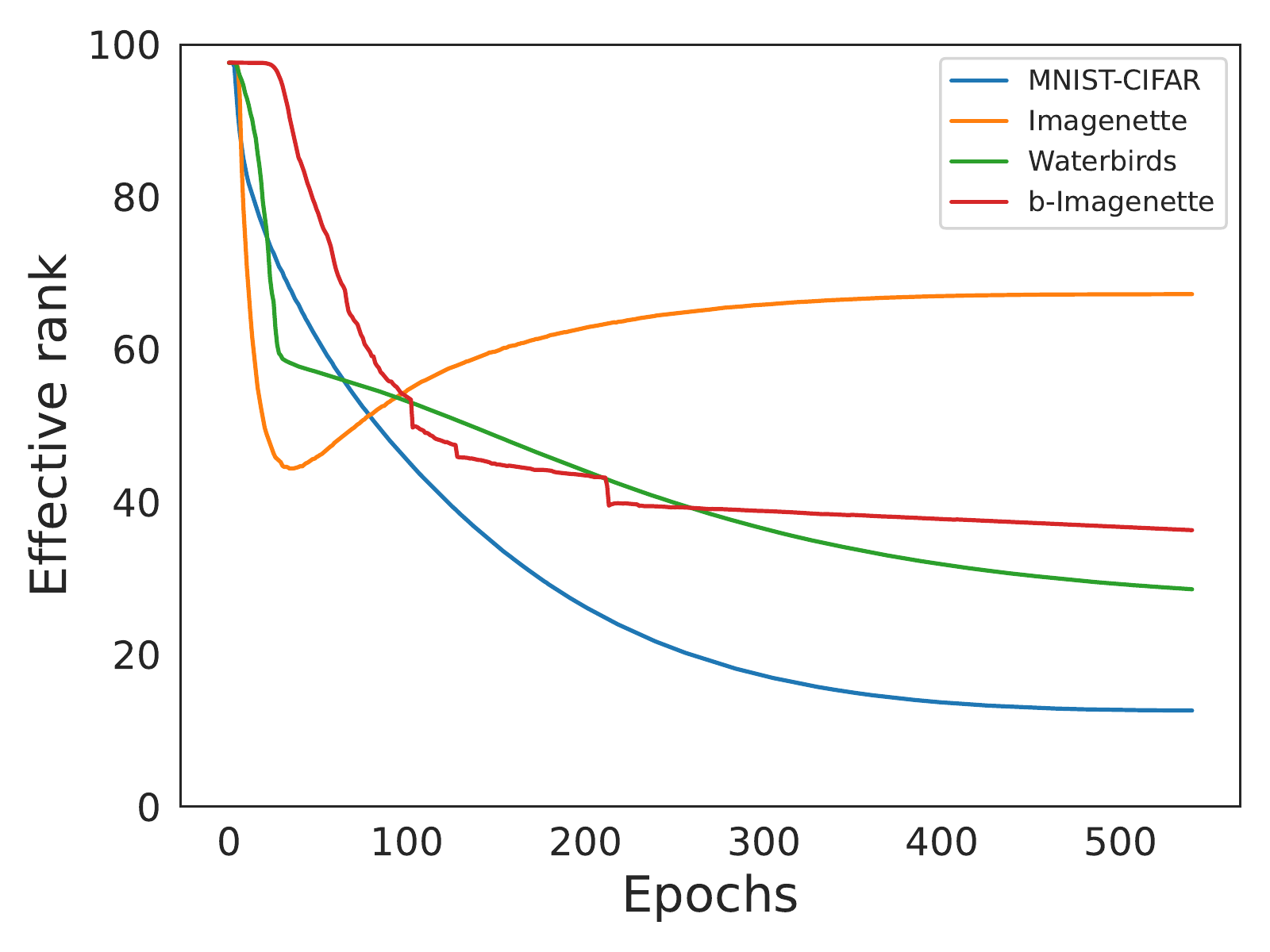}
    \caption{Layer 1}
    %\label{b-imagenet:eff_rank}
\end{subfigure}
\begin{subfigure}{.3\textwidth}
  \centering
  \includegraphics[width=\textwidth]{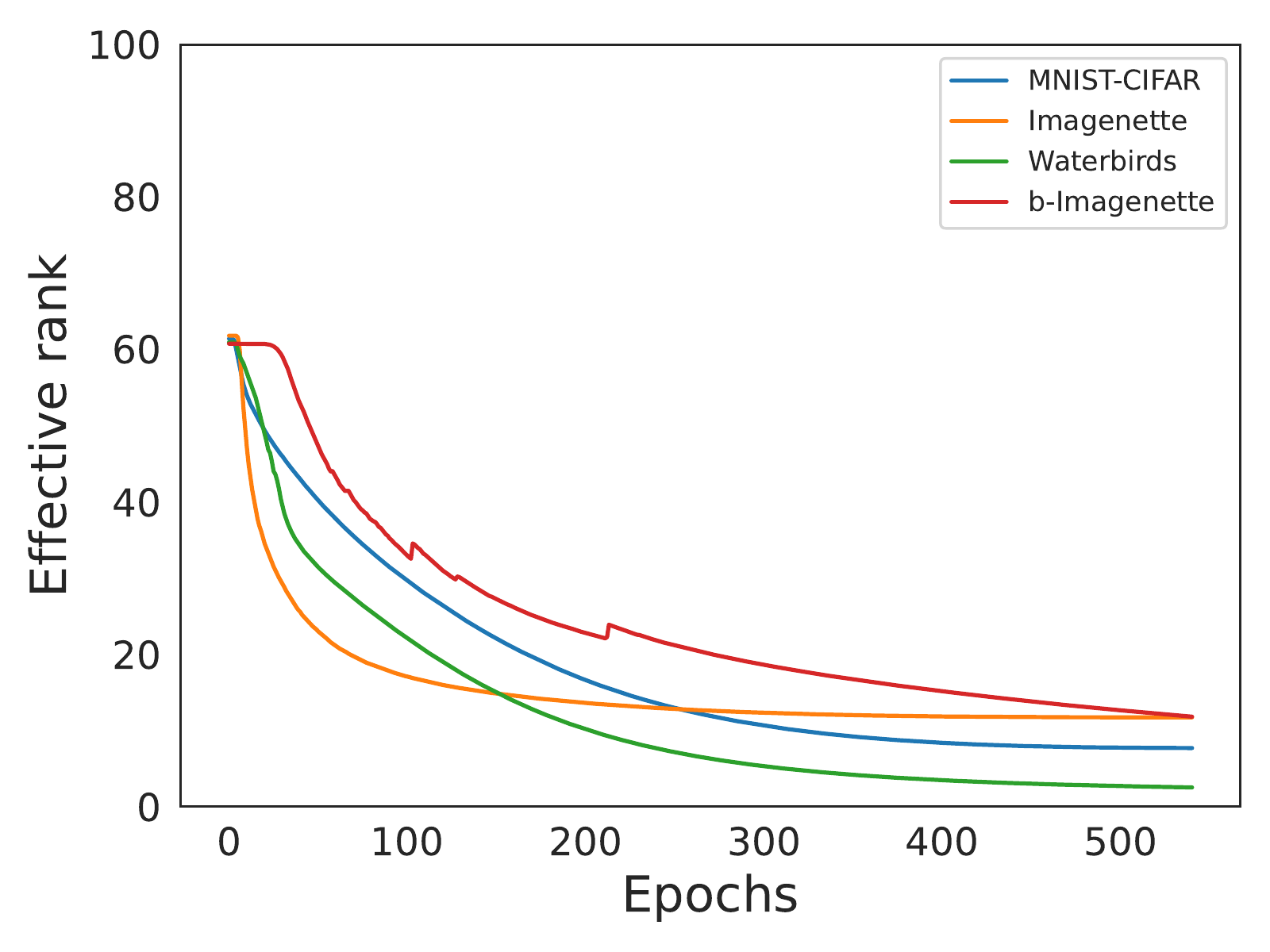}
    \caption{Layer 2}
    %\label{imagenet:eff_rank}
\end{subfigure}
\begin{subfigure}{.3\textwidth}
  \centering
  \includegraphics[width=\textwidth]{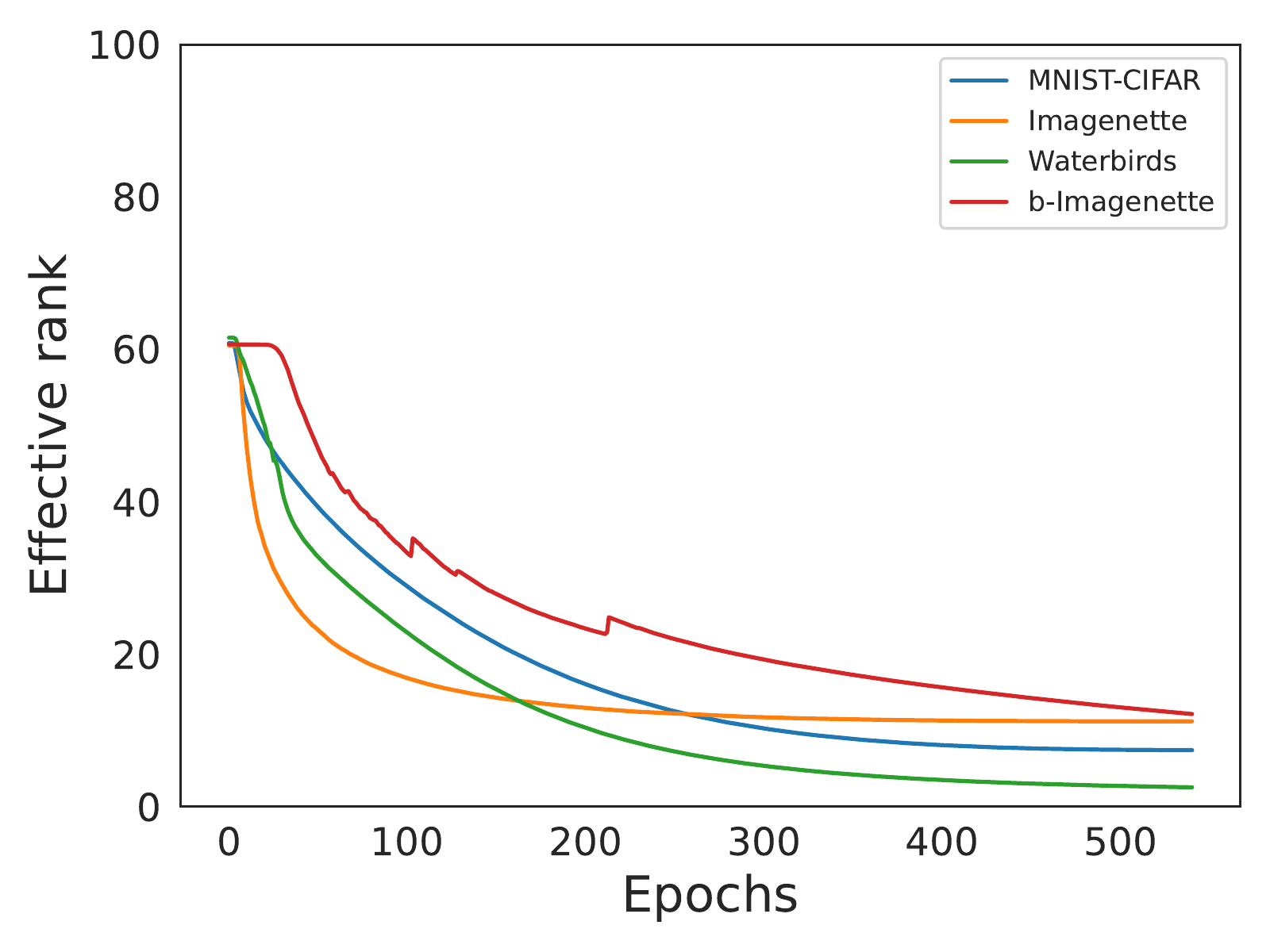}
    \caption{layer 3}
    %\label{waterbirds:eff_rank}
\end{subfigure}         
% \begin{subfigure}{.5\textwidth}
%   \centering
%   \includegraphics[width=\textwidth]{plots/mnist-cifar/gauss_robust_test_overall_3_plots.pdf}
%     \caption{MNIST-CIFAR}
%     \label{mnist-cifar:eff_rank}
% \end{subfigure}
\caption{Evolution of effective rank of the weight matrices for a depth-3 ReLU network on Resnet-50 pretrained representations of the dataset}
\label{fig:depth-3}
\end{figure}
% We show the decision boundary of $f$ and $f_{\textrm{proj}}$ on MNIST-CIFAR dataset in Figure \ref{fig:mnis_dec_bound}. As can be seen, model 2 learns a more non-linear boundary as compared to model 1.

% \begin{figure}[t]
% \centering
% \begin{subfigure}{0.4\textwidth}
%   \includegraphics[width=\textwidth]{plots/mnist-cifar/model1/decision_boundary_with_points.pdf}
%   \caption{MNIST-CIFAR ($f$)}
% \end{subfigure}
% \begin{subfigure}{0.4\textwidth}
%   \includegraphics[width=\textwidth]{plots/mnist-cifar/decision_boundary_with_points.pdf}
%     \caption{MNIST-CIFAR ($f_{\textrm{proj}}$)}
% \end{subfigure}
% \caption{Decision boundaries for $f$ and $f_\textrm{proj}$ for MNIST-CIFAR, visualized in the top $2$ singular directions of the first layer weight matrix. The decision boundary of $f_{\textrm{proj}}$ is more non-linear compared to that of $f$.}
% \label{fig:mnis_dec_bound}
% \end{figure}
\section{Extended Related Works} \label{app:rel-works}
In this section, we provide an extensive literature survey of various topics that the paper is based on.

\paragraph{Low rank Simplicity Bias in Linear Networks} Multiple works have established low rank simplicity bias for gradient descent on linear networks, both for squared loss as well as cross-entropy loss. For squared loss, \citet{gunasekar17} conjectured that the network is biased towards finding minimum nuclear norm solutions for two-layer linear networks. \citet{Arora19} refuted the conjecture and instead argued that the network is biased towards finding low rank solutions. \citet{Razin20} provided empirical support to the low rank conjecture, by providing synthetic examples where the network drives nuclear norm to infinity, but minimizes the rank of the effective linear mapping. \citet{li2021towards} established that for small enough initialization, gradient flow on linear networks follows greedy low-rank learning trajectory. For binary classification on linearly separable data, \citet{ji2019implicit} showed that the weight matrices of a linear network eventually become rank-1 as training progresses.

\paragraph{Low rank Simplicity Bias in Non-Linear Networks} For non-linear networks, the work related to low-rank simplicity bias is rather sparse. Two of the most notable works are \citet{huh2021low} and \citet{Tomer22}. \citet{huh2021low} empirically established that the rank of the embeddings learnt by a neural network with ReLU activations goes down as training progresses. \citet{Tomer22} provided an intuition behind the relation between the rank of the weight matrices and various hyperparameter such as batch size, weight decay etc. In contrast to these works, for 1 layer nets, we theoretically and empirically establish that the network depends on an extremely low dimensional projection of the input, and this bias can be utilized to develop a robust classifier.

\paragraph{Relation to OOD} Many recent works in OOD detection \citep{Cook20, Zaeemzadeh21} explicitly create low-rank embeddings so that it is easier to discriminate them for an OOD point. Other works also implicitly rely on the low-rank nature of the embeddings. \citet{Ndiour20} use PCA on the learnt features, and only model the likelihood along the small subspace spanned by the top few directions. \citet{Wang22} utilise the low rank nature of the embeddings to estimate the perpendicular projection of a given data point to this low rank subspace and combine it with logit information to detect OOD datapoints. While there have been works implicitly utilizing the low rank property of embeddings, we note that our paper (i) demonstrates low rank property of the \emph{weights}, rather than that of embeddings, and (ii) shows that it is a consequence of SB.

\paragraph{Other Simplicity Bias} There have been many works exploring the nature of simplicity bias in neural networks, both empirically and theoretically. \citet{Kalimeris19} empirically demonstrated that SGD on neural networks gradually learns functions of increasing complexity. \cite{Nasim18} empirically demonstrated that neural networks tend to learn lower frequency functions first. \cite{Ronen19} theoretically established that in NTK regime, the convergence rate depends on the eigenvalues of the kernel spectrum. \cite{Hacohen20} showed that neural networks always learn train and test examples almost in the same order, irrespective of the architecture. \cite{pezeshki2021gradient} proposes that \emph{gradient starvation} at the beginning of training is a potential reason for SB in the lazy/NTK regime but the conditions are hard to interpret. In contrast, our results are shown for any dataset in the \ifm model in the \emph{rich} regime of training. ~\cite{lyu2021gradient} consider anti-symmetric datasets and show that single hidden layer input homogeneous networks (i.e., without \emph{bias} parameters) converge to linear classifiers. However, such networks have strictly weaker expressive power compared to those with bias parameters. \cite{Guy22} showed that for deep linear networks, in NTK regime, they learn the higher principal components of the input data first. Most of the previous works used simplicity bias as a reason behind better generalization of neural nets. However, \cite{shah2020pitfalls} showed that extreme simplicity bias could also lead to worse OOD performance.

\textbf{Learning diverse classifiers}: There have been several works that attempt to learn diverse classifiers. Most works try to learn such models by ensuring that the input gradients of these models do not align~\citep{ross2018improving,teney2022evading}. \cite{xu2022controlling} proposes a way to learn diverse/orthogonal classifiers under the assumption that a complete classifier, that uses all features is available, and demonstrates its utility for various downstream tasks such as style transfer. \cite{lee2022diversify} learns diverse classifiers by enforcing diversity on unlabeled target data.

\textbf{Spurious correlations}: There has been a large body of work which identifies the reasons for spurious correlations in NNs~\citep{sagawa2020investigation} as well as proposing algorithmic fixes in different settings~\citep{liu2021just,chen2020self}.

\textbf{Implicit bias of gradient descent}: There is also a large body of work understanding the implicit bias of gradient descent dynamics. Most of these works are for standard linear~\citep{ji2019implicit} or deep linear networks~\citep{SoudryHNGS18,gunasekar2018implicit}. For nonlinear neural networks, one of the well-known results is for the case of $1$-hidden layer neural networks with homogeneous activation functions~\citep{ChizatB20}, which we crucially use in our proofs.

\section{More discussion on the extension of results to deep nets} \label{app:deep-nets}
Extending our theoretical results to deep nets is a very exciting and challenging research direction. For shallow as well as deep nets, even in the mean field regime of training, results regarding convergence to global minima have been established \citep{Chizat18, Cong21}. However, to the best of our knowledge, only for 1-hidden layer FCN \citep{ChizatB20}, a precise characterization of the global minima to which gradient flow converges has been established. Understanding this implicit bias of gradient flow is still an open problem for deep nets, which we think is essential for extension of our results to deep nets.

\section{Convergence to $\F_1-$max-margin classifier for ReLU networks} \label{app:relu-F1}
In this section, we will state the precise result of \citet{ChizatB20} regarding the asymptotic convergence point of gradient flow on ReLU networks. We will follow the notation of \citet{ChizatB20} for ease of the reader.

A neural network is parameterized by a probability measure $\mu$ on the neurons and is given by
\[ h(\mu, x) = \int \phi(w, x) d\mu(w) \]
where $\phi(w, x) = b (a^\top(x,1))_+$ ($+$ denotes the positive component, i.e the ReLU activation) with $w=(a,b) \in \R^{d+2}$. As the network is 2-homogeneous, a projection of the measure $\mu$ on the unit sphere can be defined. The projection operator ($\Pi_2$) on the sphere for a measure $\mu$ is defined such that for any continuous function $\varphi$ on the sphere,
\[ \int_{\mathbb{S}^{d+1}} \varphi(\theta) d[\Pi_2(\mu)](\theta) = \int_{\R^{d+2}} \|w\|^2 \varphi(w/\|w\|) d\mu(w) \]

Now, let $\rho$ denote the input distribution on the input space $\mathcal{X}$ and let the labeling function $y:\mathcal{X} \to \mathcal{Y}$ be deterministic. Then, consider the population objective given by
\[ F(\mu) = -\log\left[ \int_{\mathcal{X}} \exp(-y(x)h(\mu,x)) d\rho(x) \right] \]
Note that $\log$ doesn't affect the direction of the gradients, thus, the trajectory of gradient flow on this loss is the same as on exponential loss. Also, let the population smooth margin be given by
\[ S(f) = -\log\left(\int_{\mathcal{X}} \exp(-f(x)) d\rho(x)\right) \]
For this particular case, $f(x) = y(x)h(\mu,x)$. Denote $y(x)\cdot h(\mu,x)$ by $\hat{h}(\mu)$.
% Now, under the assumption that $\rho$ has bounded density and bounded support, and labeling function $y$ is continuous, the following holds:
\begin{theorem}\label{thm:technical}
Suppose that $\rho$ has bounded density and bounded support, and labeling function $y$ is continuous, then
there exists a Wasserstein gradient flow $(\mu_t)$ on $F$ with $\mu_0 = \mathcal{U}(\mathbb{S}^d) \otimes \mathcal{U}\{-1,1\}$, i.e, input (resp. output) weights uniformly distributed on the sphere (resp. on $\{-1,1\}$). If $\nabla S(\hat{h}(\mu_t))$ 
% \red{What is $\hat{h}$? It is $y(x)h(\mu,x)$, I have specified it just above the theorem} 
converges weakly in $\mathcal{P}(\mathcal{X})$, if $\bar{\nu}_t = \Pi_2(\mu_t)/([\Pi_2(\mu_t)](\mathbb{S}^{d+1}))$ converges weakly in $\mathcal{P}(\mathbb{S}^{d+1})$ and  $F'(\mu_t)$ converges in $C_{loc}^1$ to $F'$ that satisfies the Morse-Sard property, then $h(\bar{\nu}_\infty,.)$ is a maximizer for $\max_{\|f\|_{\mathcal{F}_1} \leq 1} \min_{x \in \mathcal{X}} y(x)f(x)$.
\end{theorem}
where $\mathcal{P}(\mathcal{X})$ denotes the space of probability distributions on $\mathcal{X}$ and $[\Pi_2(\mu_t)](\mathbb{S}^{d+1})$ denotes the total mass of the measure $\Pi_2(\mu_t)$ on $\mathbb{S}^{d+1}$. 

To parse the theorem, note that
\[ \nabla S(f) = \frac{\exp(-f(x)) d\rho(x)}{\int_\mathcal{X} \exp(-f(x')) d\rho(x')} \]
Thus, $\nabla S(f)$ convergence means that the exponentiated normalized margins converge. Also, $\bar{\nu}_t$ is similar to the directional convergence of weights, however, in this case, weights are replaced by directions in $\mathbb{S}^{d+1}$. For explanation of the Morse-Sard property and the metric $C_{loc}^1$, please refer to Appendix H of \citet{ChizatB20}.

\end{document}

% This document was modified from the file originally made available by
% Pat Langley and Andrea Danyluk for ICML-2K. This version was created
% by Iain Murray in 2018, and modified by Alexandre Bouchard in
% 2019 and 2021 and by Csaba Szepesvari, Gang Niu and Sivan Sabato in 2022.
% Modified again in 2023 by Sivan Sabato and Jonathan Scarlett.
% Previous contributors include Dan Roy, Lise Getoor and Tobias
% Scheffer, which was slightly modified from the 2010 version by
% Thorsten Joachims & Johannes Fuernkranz, slightly modified from the
% 2009 version by Kiri Wagstaff and Sam Roweis's 2008 version, which is
% slightly modified from Prasad Tadepalli's 2007 version which is a
% lightly changed version of the previous year's version by Andrew
% Moore, which was in turn edited from those of Kristian Kersting and
% Codrina Lauth. Alex Smola contributed to the algorithmic style files.